\documentclass{article} %

\usepackage{hyperref}
\usepackage{url}
\usepackage{booktabs}       %
\usepackage{amsfonts}       %
\usepackage{nicefrac}       %
\usepackage{microtype}      %

\usepackage{amsmath,amsfonts,amsthm}       %
\usepackage{subcaption}
\usepackage{bm}
\usepackage{bbm}
\usepackage{multirow}
\usepackage{multicol}
\usepackage[inline]{enumitem}
\usepackage{diagbox}
\usepackage[capitalise]{cleveref}  %
\usepackage{comment}
\usepackage{etoolbox}
\usepackage{graphicx}
\usepackage{wrapfig}
\usepackage[font=small]{caption}
\usepackage{algorithm,algorithmicx,algpseudocode}
\usepackage{subcaption}
\usepackage{sidecap}

\usepackage{pifont}%

\usepackage{xcolor}
\usepackage{adjustbox}
\usepackage{colortbl}
\usepackage{rotating}
\usepackage{threeparttable}
\usepackage{xspace}
\usepackage{mdframed}
\usepackage[utf8]{inputenc} 
\usepackage[T1]{fontenc}    
\newcommand{\icon}{\raisebox{-14pt}{\includegraphics[width=3.0em]{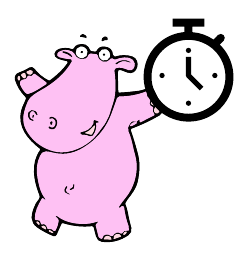}}\xspace}
\hypersetup{
    colorlinks=true,
    linkcolor=blue,
    urlcolor=blue,
    citecolor=red
}

\newtheorem{theorem}{Theorem}

\newtheorem{lemma}{Lemma}[section]
\newtheorem{corollary}[lemma]{Corollary}

\newtheorem{proposition}[theorem]{Proposition}
\newtheorem{definition}{Definition}
\newtheorem{remark}[lemma]{Remark}

\newcommand{\dt}{\Delta}

\newcommand{\TS}{Time-SSM}

\newcommand{\para}[1]{\paragraph{#1}}

\usepackage[square,numbers,sort]{natbib}
\newlength{\defbaselineskip}
\definecolor{mycolor1}{RGB}{72,86,126}
\definecolor{mycolor2}{RGB}{142,171,192}
\definecolor{mycolor3}{RGB}{199,181,160}
\definecolor{mycolor4}{RGB}{183,191,153}
\definecolor{mypurple}{RGB}{128,0,128}

\title{ \icon \hspace{-5pt}Time-SSM: Simplifying and Unifying State Space Models for Time Series Forecasting}

\usepackage{authblk}
\author[$\dagger$]{Jiaxi Hu}
\author[$\dagger$]{Disen Lan}
\author[$\dagger$]{Ziyu Zhou}
\author[$\ddagger$]{Qingsong Wen}
\author[$\dagger *$]{Yuxuan Liang}
\affil[$\dagger$]{Hong Kong University of Science and Technology (Guangzhou)}
\affil[$\ddagger$]{Squirrel AI}
\affil[$\dagger$]{{\texttt{jiaxihu@hkust-gz.edu.cn}, \texttt{yuxliang@outlook.com}}}
\date{}

\begin{document}

\maketitle

\begin{abstract}

State Space Models (SSMs) have emerged as a potent tool in sequence modeling tasks in recent years. These models approximate continuous systems using a set of basis functions and discretize them to handle input data, making them well-suited for modeling time series data collected at specific frequencies from continuous systems. Despite its potential, the application of SSMs in time series forecasting remains underexplored, with most existing models treating SSMs as a black box for capturing temporal or channel dependencies. To address this gap, this paper proposes a novel theoretical framework termed Dynamic Spectral Operator, offering more intuitive and general guidance on applying SSMs to time series data. Building upon our theory, we introduce Time-SSM, a novel SSM-based foundation model with only one-seventh of the parameters compared to Mamba. Various experiments validate both our theoretical framework and the superior performance of Time-SSM. 

\end{abstract}

\section{Introduction}
State Space Models (SSMs) \cite{gu2021efficiently,gu2023mamba,gu2022parameterization,smith2022simplified} are recent deep learning models that have demonstrated significant potential in various sequence modeling tasks. Based on controlled differential equations and discretization techniques, SSMs naturally align with time series data, which are meticulously collected at predetermined sampling frequencies from underlying continuous-time dynamical systems. 
Recently, a series of models \cite{patro2024simba,atik2024timemachine,wang2024mamba,behrouz2024mambamixer,zhang2023effectively,hu2024attractor} encapsulate the state-of-the-art variants (e.g., Mamba \cite{gu2023mamba}) as black-box modules, attempting augmenting their temporal modeling prowess.

Despite their promising performance, the intricate mathematical foundations and the underlying physical interpretations of SSMs in the realm of time series forecasting (TSF) \cite{liang2024foundation,jin2023large} remain ambiguous. For instance,
(i) SSM is mathematically interpreted as Generalized Orthogonal Basis Projection \cite{gu2022train} (GOBP), representing a linear combination of convolution kernels parameterized by two matrices ($\mathbf{A}$ and $\mathbf{B}$) and controlled by a matrix $\mathbf{C}$. However, this formulation falls short of interpreting complex temporal dependencies when generalizing to time series. 
(ii) SSMs originate from $\mathsf{Hippo}$ \cite{gu2020hippo}, and the physical significance of various initialization in the field of TSF remains unclear. 
(iii) SSMs are initially devised for modeling ultra-long sequences (e.g., with a target of one million in $\mathsf{Hippo}$), which prompts the adoption of an exponentially decaying basis, i.e., \textit{S4-LegS} \cite{gu2022train}. Nevertheless, the necessity of this decaying method for TSF tasks requires further justification. 
This work aims to provide a comprehensive theoretical exposition of the above questions, offering more intuitive and general guidance on applying SSMs to time series data. As immediate consequences of this paper:

\begin{itemize}[leftmargin=*]
    \item\textbf{\textit{Clarity and Unity Spark Insight}}. We firmly believe that a clear and unified theoretical framework will directly foster the advancement of SSMs in the TSF community. To this end, we extend the Generalized Orthogonal Basis Projection (GOBP) theory to the Dynamic Spectral Operator theory (see Figure~\ref{fig1}), which is more appropriately tailored for TSF tasks. Building upon this, we introduce a novel variant of SSM basis called \textbf{\textit{Hippo-LegP}} based on piecewise Legendre polynomials, along with various time-varying SSMs in complex-plane for temporal dynamic modeling.

    \item\textbf{\textit{Never Train From Scratch}}. 
    Early SSMs are initialized based on matrices derived from $\mathsf{Hippo}$ theory \cite{gu2020hippo}, later evolving into negative real diagonal initialization \cite{gu2023mamba}. This paper elucidates the physical significance of diagonal initialization from the perspective of time series dynamics and demonstrates that using specific matrices for initialization remains optimal in the context of TSF.

    \item\textbf{\textit{A Dawn for Time Series}}. Leveraging the above supports, we introduce \textbf{\TS}, a novel SSM-based foundation model for TSF. 
    Various experiments, including SSM variant ablation, autoregressive prediction, and long-term function reconstruction, empirically validate our theoretical framework and the superior performance of \TS. We aim for this to mark the dawn of SSM's application in time series data, catalyzing vibrant growth and innovation within this community.
\end{itemize} 

\begin{mdframed}[backgroundcolor=gray!20]
\begin{minipage}{\linewidth}
Apart from theoretical contribution, the important empirical findings include:

1. A unique phenomenon observed in the time series area that finite measure SSM and robust SSM demonstrate superior performance, in contrast to the results in Long Range Arena benchmarks.

2. S4D-real and LegP-complex initializations outperform others, and S4D-real is particularly efficient.

3. For complex-plane SSMs, the unitary matrix $\bm{V}$ serves as an inducing bias that enables $\bm{ABC}$ to possess coordinated dynamics, which is important for temporal modeling.

4. Preserving the specific initialization of $\bm{A}$ is crucial in SSM-based TSF models.

5. Autoregressive prediction is unnecessary in SSM-based models, LegP exhibits stronger representational capacity but is susceptible to noise in TSF tasks, the linear weights in Time-SSM play a crucial role, time-invariant $\bm{B}$ is beneficial for long-term prediction, and so on.




\end{minipage}
\end{mdframed}

\section{Theoretical Framework} \label{background}
\begin{figure}[t]
\centering
{\includegraphics[width=1\columnwidth]{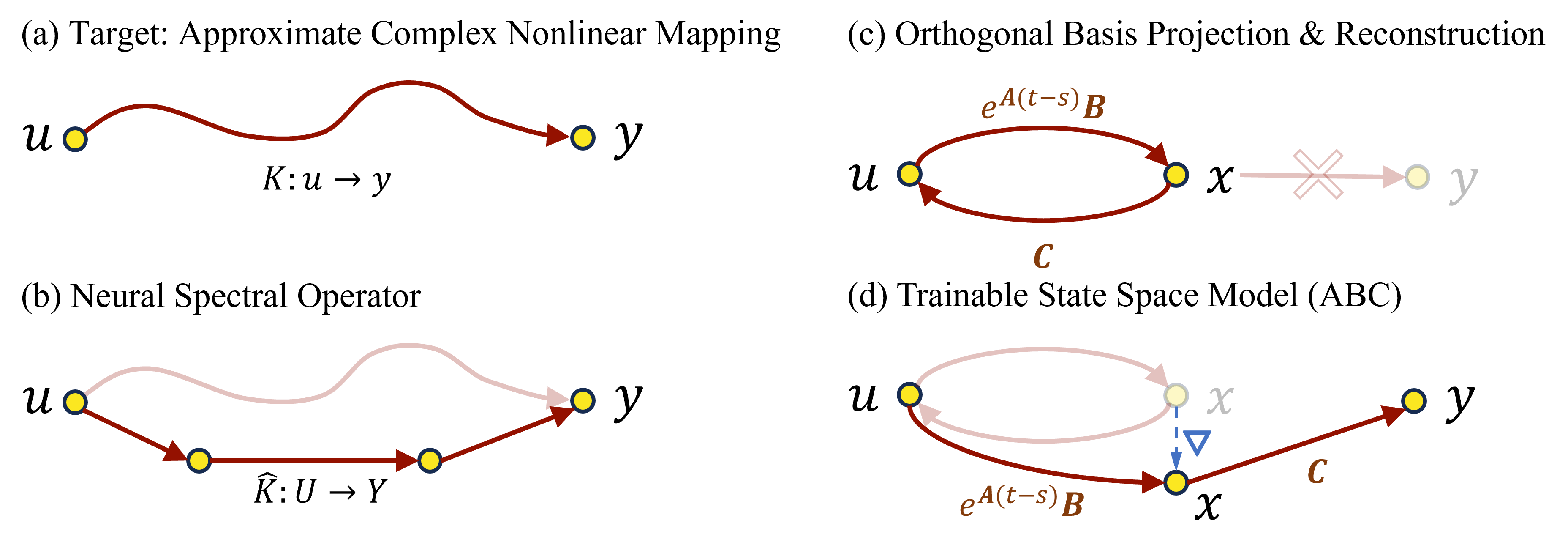}}
\caption{From generalized orthogonal basis projection theory to dynamic spectral operator theory. (a) In the TSF task, we aim to find a complex nonlinear mapping $\mathcal{K}$ from past observed data $u$ to future predicted data $y$. (b) Neural spectral operators simplify the parameterization of mapping $\mathcal{K}$ by leveraging an easily learnable spectral space. (c) The GOBP theory represents the projection and reversible reconstruction of input data $u$ to the spectral space $x$. (d) We define the learnable SSMs as a dynamic operator that gradually shifts from an initial spectral projection space toward an optimal spectral space and find a linear projection to the observation $y$.}
\label{fig1}
\end{figure}
We will start by presenting the neural spectral operator and continuous-time SSM from a time series perspective, and then define SSM as a \textbf{Dynamic Spectral Operator} that adaptively explores temporal dynamics based on our proposed \textbf{IOSSM} concept. \underline{All proofs can be found in Appendix \ref{proof}}.

\para{Neural Spectral Operator.}
Given two functions $u(s)$ and $y(t)$, the operator is a map $\mathcal{K}$ such that $\mathcal{K}u=y$ in Sobolev spaces $\mathcal{H}^{b, d}$. Typically, we choose $b>0$ and $d=2$ to facilitate the definition of projections w.r.t. measures $\mu$ in a Hilbert space structure.
We take the operator $\mathcal{K}:u\mapsto y$ as an integral operator in interval $D$ with the kernel $K(t,s)$ as Equation \eqref{operator}, referring to Figure \ref{fig1}a. 

It is apparent that learning complete kernel $K(t,s)$
would essentially solve the operator map problem. However, it is not necessarily a numerically feasible solution. An overarching approach is to obtain a compact representation $\hat{K}$ of the complete sparse kernel $K$, which necessitates one or a sequence of suitable domains along with projection tools $\mathbbm{T}$ to these domains as Equation \eqref{spectral operator} (see Figure \ref{fig1}b).

\begin{minipage}{.5\linewidth}%
\begin{equation}
\mathcal{K}u(s)=\int_D K(t, s) u(s) d s,
\label{operator}
\end{equation}
\end{minipage}
\begin{minipage}{.5\linewidth}%
\begin{equation}
K(t, s) =\mathbbm{T} \circ \hat{K}(t, s)\circ \mathbbm{T}^{-1}.
\label{spectral operator}
\end{equation}
\end{minipage}
\normalsize


This paper refers to these domains as spectral domains, such as Fourier \cite{li2020fourier}, Laplace \cite{cao2023lno}, and Wavelet \cite{gupta2021multiwavelet} domain, where operators can be directly parameterized using either linear layers or simple neural networks.
In the realm of TSF, the operator can be abstracted as a mapping from a function $u(s)$ about historical time to a function $y(t)$ about future time.
\para{Continuous-Time SSM.}
Defined by Equation \eqref{eq:ssm}, SSMs are parameterized maps that transform the input $u(t)$ into an $N$-dimensional latent space and project it onto the output $y(t)$. In this paper, we focus on single-input/output (\textbf{SISO} \cite{gu2021efficiently}) SSM for simplicity, and our theoretical framework can be easily extended to multi-input/output (\textbf{MIMO} \cite{smith2022simplified}) SSM with input $u\in D$ dimension.

\begin{minipage}{.49\linewidth}
\begin{subequations}\label{eq:ssm}
\begin{align}
  \label{eq:ssm-a}
  x'(t) &= \bm{A}x(t) + \bm{B}u(t)  \\
  \label{eq:ssm-b}
  y(t)  &= \bm{C}x(t)
\end{align}
\end{subequations}
\end{minipage}
\begin{minipage}{.49\linewidth}
\begin{subequations}\label{eq:ssm-conv}
\begin{align}
\label{eq:ssm-conv-a}
  K(t) &= \bm{C} e^{t\bm{A}} \bm{B} \\
  \label{eq:ssm-conv-b}
  y(t) &= (K \ast u)(t) %
\end{align}
\end{subequations}
\end{minipage}

The two terms in Lemma \ref{general solution} are also known as \textit{zero-input response} component and \textit{zero-state response} component \cite{williams2007linear}. When neglecting the influence of the initial state on dynamic evolution, specifically by setting $x(t_0)=0$, SSM can be represented as a convolution kernel as Equation \eqref{eq:ssm-conv}.
\begin{lemma} 
    For a differential equation of $x^{\prime}(t)=\bm{A} x(t)+\bm{B} u(t)$, its general solution is:
$$
x(t)=e^{\bm{A}\left(t-t_0\right)} x\left(t_0\right)+\int_{t_0}^t e^{\bm{A}(t-s)} \bm{B} u(s) \mathrm{d}s.
$$
\label{general solution}
\end{lemma}
\vspace{-1em}
$\mathsf{Hippo}$ \cite{gu2020hippo} provides a mathematical framework for deriving the $\bm{AB}$ matrix: Given an input $u(s)$, a set of closed-recursive orthogonal basis $p_n(t,s)$ that $\int_{-\infty}^t p_m(t,s)p_n(t,s)\mathrm{d}s=\delta_{m,n}$, and an inner product probability measure $\mu(t,s)$. This enables us to project the input $u(s)$ onto the basis along the time dimension in Hilbert space $\mathcal{H}_\mu$, as shown in Equation \eqref{hippo}. When we differentiate this projection coefficient w.r.t. $t$ (Equation \eqref{eq:ssm-a}), the self-similar relation \cite{gu2020hippo} allows $\frac{\mathrm{d}}{\mathrm{d} t} x_n(t)$ to be expressed as a Linear ODE in terms of $x_n(t)$ and $u(s)$, with temporal dynamics determined by input $u(s)$.

\begin{minipage}{.5\linewidth}%
\begin{equation}
\langle u, p_n\rangle_{\mu} = \int_{-\infty}^t u(s)p_n(t,s)\omega(t,s)\mathrm{d}s,
\label{hippo}
\end{equation}
\end{minipage}
\begin{minipage}{.5\linewidth}%
\begin{equation}
x_n(t) = \int u(s)K_n(t,s)\mathbbm{I}(t,s)\mathrm{d}s.
\label{ssm kernel}
\end{equation}
\end{minipage}
\normalsize

From Equation (\ref{hippo}), it is evident that the projection coefficients depend on both the terms $\omega(t,s)$ and $p_n(t,s)$. The GOBP theory \cite{gu2022train} combines them into a single term $K_n(t,s)$, called SSM kernel, and constrains the integration interval using indicator functions $\mathbbm{I}(t,s)$ (see Equation \ref{ssm kernel}).

\begin{lemma} \label{nonlinear}
    Any nonlinear, continuous differentiable dynamic $\dot{x}(t) =f(x(t), u(t), t)$ can be represented by its linear nominal SSM ($\tilde{x}, \tilde{u}$) plus a third-order infinitesimal quantity:
$$
\dot{x}(t)=\bm{A} x(t)+\bm{B}u(t) +\mathcal{O}(x,u); ~~~~~ \bm{A}=\left.\frac{\partial f}{\partial x}(x, u)\right|_{\tilde{x}, \tilde{u}}; \bm{B}=\left.\frac{\partial f}{\partial u}(x, {u})\right|_{\tilde{x}, \tilde{u}}.
$$
\end{lemma}
Lemma \ref{nonlinear} provides a theoretical guarantee in modeling non-linear and non-stationary temporal dependencies of real-world time series by linear SSM (Equation \eqref{eq:ssm}) with limited error.
\begin{definition}
We call SSM system an \textbf{invertible orthogonal SSM (IOSSM)}, with basis $p_n(t,s)$ and measure $\omega(t,s)$ in Hippo($\bm{AB}$) and corresponding reconstruction matrix $\bm{C}$ (Figure \ref{fig1}c), at all time t,
 $$\textbf{IOSSM}(\bm{AB}): u(s)\mapsto x(t), \quad\quad\quad\quad \textbf{IOSSM}(\bm{C}): x(t)\mapsto u(s).$$
\label{def: Issm}
\end{definition}
\vspace{-2em}
\begin{remark}
    This paper also adheres to other definitions in HTTYH (GOBP) \cite{gu2022train}, such as OSSM (Orthogonal SSM) and TOSSM (Time-invariant Orthogonal SSM).
\end{remark}
Previous SSMs treat matrix $\bm{AB}$ and $\bm{C}$ in isolation, and $\bm{C}$ is typically understood as coefficients for linear combinations of kernel $e^{t\bm{A}}\bm{B}$ \cite{gu2022train} or as a feature learning module \cite{gu2021combining}. However, these interpretations do not align well with $\mathsf{Hippo}$ theory, which prioritizes the function projection and reconstruction rather than the mapping from input $u$ to output $y$. To address this discrepancy, we propose the concept of {IOSSM}, as stated in Definition \ref{def: Issm}. IOSSM defines a spectral transformation $\mathbbm{T}$ and inverse transformation $\mathbbm{T}^{-1}$ by IOSSM($\bm{ABC}$) in formula like \eqref{eq:hippo-legs} and \eqref{eq:hippo-legt}. 

\small
\begin{minipage}{.49\linewidth}%
\begin{equation}
  \label{eq:hippo-legs}
  \qquad
  \begin{aligned}%
    \bm{A}_{nk}
    &=
    -
    \begin{cases}
      (2n+1)^{\frac{1}{2}}(2k+1)^{\frac{1}{2}} & n > k \\
      n+1 & n = k \\
      0 & n < k
    \end{cases}
    \\
    \bm{B}_n &= (2n+1)^{\frac{1}{2}}
    \quad
    \bm{C}_n = (2n+1)^{\frac{1}{2}} L_{n}(2n -1)
    \\
    & (\textbf{HiPPO-LegS matrix in IOSSM})
  \end{aligned}
\end{equation}
\end{minipage}
\hspace{10pt}
\begin{minipage}{.49\linewidth}%
\begin{equation}
  \label{eq:hippo-legt}
  \begin{aligned}%
    & \bm{A}_{nk}=
  -(2n+1)^{\frac{1}{2}}(2k+1)^{\frac{1}{2}}\\
  &\cdot
  \begin{cases}
    1 & k \le n \\
    (-1)^{n-k} & k \ge n
  \end{cases}
  \\
  \bm{B}_n &= (2n+1)^{\frac{1}{2}}
  \quad
  \bm{C}_n = L_{n}(2n -1)
    \\
    & (\textbf{HiPPO-LegT matrix in IOSSM})
  \end{aligned}
\end{equation}
\end{minipage}
\begin{minipage}{.49\linewidth}%
\begin{equation}
  \label{eq:hippo-legsd}
  \begin{aligned}%
    \bm{A}^{(N)}_{nk}
  &=
    -
    \begin{cases}
      (n+\frac{1}{2})^{1/2}(k+\frac{1}{2})^{1/2} & n > k \\
      \frac{1}{2} & n = k \\
      (n+\frac{1}{2})^{1/2}(k+\frac{1}{2})^{1/2} & n < k
    \end{cases}
    \\
    \bm{A} &= \bm{A}^{(N)} - \operatorname*{rank}(1),
    \qquad
    \bm{A}^{(D)} := \operatorname*{eig}(\bm{A}^{(N)})
    \\
    & (\textbf{Normal / DPLR form of HiPPO-LegS})
  \end{aligned}
\end{equation}
\end{minipage}
\hspace{10pt}
\begin{minipage}{.49\linewidth}%
\begin{equation}
  \label{eq:hippo-legtd}
  \begin{aligned}%
    \bm{A}^{(N)}_{nk}
  &=
    -
    \begin{cases}
      (2n+1)^{\frac{1}{2}}(2k+1)^{\frac{1}{2}} & n < k, k~odd \\
      0 & else \\
      (2n+1)^{\frac{1}{2}}(2k+1)^{\frac{1}{2}} & n > k, n~odd
    \end{cases}
    \\
    \bm{A} &= \bm{A}^{(N)} - \operatorname*{rank}(2),
    \qquad
    \bm{A}^{(D)} := \operatorname*{eig}(\bm{A}^{(N)})
    \\
    & (\textbf{Normal / DPLR form of HiPPO-LegT})
  \end{aligned}
\end{equation}
\end{minipage}
\normalsize
\looseness=-1
\para{Dynamic Spectral Operator.}
Based on definitions of SSM Kernel and IOSSM, we can establish a connection to the commonly used spectral transformation techniques in time series analysis.
\begin{figure}[t]
\begin{adjustbox}{width=1\columnwidth, center}
\centering
\includegraphics[width=1\columnwidth]{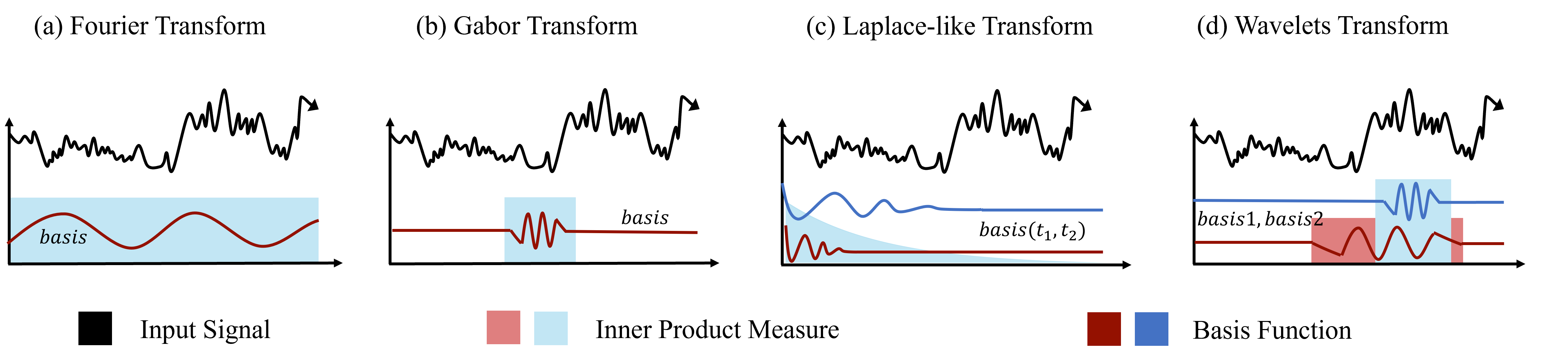}
\end{adjustbox}
\caption{Classical spectral transform with basis function and inner product measure.}
\label{SSM as TS transform}
\end{figure}
\begin{proposition} \label{fourier}The IOSSM($\bm{A}$,$\bm{B}$) is a Fourier Transform (Figure \ref{SSM as TS transform}a) with
    $$\omega(t,s) = \mathbbm{1}_{[0,1]}, \qquad \qquad \qquad p_n(t, s) =  e^{-its}.$$
\end{proposition}
This particular system effectively utilizes the properties of trigonometric functions to represent the periodic characteristics of a time series in the spectral domain. Nonetheless, $\omega(t)$ and $p_n(t, s)$ are defined across the entire domain, making it challenging to capture the local features.
\begin{proposition} The IOSSM($\bm{A}$,$\bm{B}$) is a Gabor Transform (Figure \ref{SSM as TS transform}b) with
$$\omega(t,s) = \mathbbm{1}_{[0,1]}, \qquad \qquad \qquad p_n(t, s) =  g(t-\tau)\phi_n(t,s).$$
\end{proposition}
The Gabor function $g(t-\tau)$ defines a window to truncate any orthogonal basis $\phi_n(t,s)$ and indirectly influence the $\omega(t)=\mathbbm{I}[t,t+\tau]$, enabling the system to capture the local dynamic characteristics. \textbf{\textit{HiPPO-LegT}} \eqref{eq:hippo-legt} use translated Legendre basis ${L_n}$, \textbf{\textit{S4-FouT}} \cite{gu2022train} utilize Fourier basis, known as the short-time Fourier transform, and Convolution Neural Network can be regarded as learnable basis.
\begin{proposition} The IOSSM($\bm{A}$,$\bm{B}$) is a Laplace Transform with
    $$\omega(t,s) = e^{-\sigma t}, \qquad \qquad \qquad p_n(t, s) =  e^{-its}.$$
\end{proposition}
Exponential decay measures benefit long-term dynamics modeling. In the case of $\textbf{\textit{S4-LegS}}(\bm{A}$,$\bm{B})$, a variant of $\textbf{\textit{Hippo-LegS}}(\frac{1}{t}\bm{A},\frac{1}{t}\bm{B})$ \eqref{eq:hippo-legs} without the $\frac{1}{t}$ term, it resembles a Laplace-like transform (Figure \ref{SSM as TS transform}c) with exponential decay in both the measure and the Legendre basis function \cite{gu2022train}.
\begin{proposition} The IOSSM($\bm{A}$,$\bm{B}$) is a Wavelet Transform (Figure \ref{SSM as TS transform}d) with
    $$\omega(t,s) = \mathbbm{1}_{[0,1]}, \qquad \qquad \qquad p_n(t, s) =  {|m|}^{-\frac{1}{2}}\phi_{n}(\frac{t-\tau}{m}).$$
\end{proposition}
In this scenario, the basis functions are scaled using parameter $m$ and shifted by parameter $\tau$, enabling flexible modeling of the dynamic characteristics of time series at multiple scales. This approach is particularly well-suited for capturing the non-stationary distribution properties of real-world time series. Appendix \ref{visual} offers an intuitive visual explanation of these concepts.

\looseness=-1
\para{HiPPO-LegP.} We set $p_n^{m\tau}(t, s) =  2^{m/2}L_n(2^{m}t-\tau)$ to obtain an SSM system based on piecewise Legendre polynomial basis, with segmentation performed at powers of 2. Lemma \ref{mean error} provides the approximation error of it with a finite basis. Appendix \ref{compute legp} presents a detailed implementation.

\begin{lemma}\cite{hu2024attractor} 
Suppose that the function $u:[0,1] \in \mathbbm{R}$ is $k$ times continuously differentiable, the piecewise polynomial $g\in\mathcal{G}_{r}^k$ approximates $u$ with mean error bounded as follows:
\small
$$
\left\|u-g\right\| \leq 2^{-r k} \frac{2}{4^k k !} \sup _{x \in[0,1]}\left|u^{(k)}(x)\right|.
$$
\label{mean error}
\end{lemma}

\begin{corollary} The SSM based on \textbf{\textit{HiPPO-LegP}} have stronger dynamic representation ability.
\label{legp good}
\end{corollary}

\begin{remark}
    \textbf{\textit{HiPPO-LegP}} is a framework that can be extended to an arbitrary orthogonal basis. Since the piecewise-scale projection is a linear projection $\mathbbm{T}^{2}=\mathbbm{T}$ onto a subspace, preserving the system dynamic characteristics determined by $\bm{AB}$. Therefore, the $\bm{AB}$ matrix in \textbf{\textit{HiPPO-LegP}} is consistent with that of \textbf{\textit{HiPPO-LegT}} but with additional fixed or trainable $\mathbbm{T}$ and $\mathbbm{T}^{-1}$.
\end{remark}

Based on spectral neural operator and continuous-time SSM, we propose Dynamic Spectral Operator (\textbf{Definition} \ref{def: dynamic operator}). We call it "dynamic" because conventional integral operators focus on global mappings, whereas SSM can focus on dynamic integral transformations by altering the basis and measure w.r.t. time, facilitating an adaptive shift to the spectral space best suited for capturing temporal dynamics.

\begin{definition} SSM is a \textbf{Dynamic Spectral Operator} $\mathcal{K}:u\mapsto y$ by transform $\mathbbm{T}$ with integral kernel $e^{\bm{A}(t-s)}$. $\bm{C}$ can be regarded as $\hat{K}\circ\mathbbm{T}^{-1}$ or directly linear projection: $:x\mapsto y$ (Figure \ref{fig1}d).
\label{def: dynamic operator}
\end{definition}

\begin{proposition}\label{space skewing}The process of gradient updates in the IOSSM($\bm{ABC}$) is a spectral space skewing:
$$span\{\phi_i\} \rightarrow span\{\psi_i\}, \quad \psi_i=\chi(t,s)\circ\phi_i.$$
\end{proposition}

\looseness=-1
\para{Diagonal SSM.}
In practical scenarios, the naive recursive calculation of SSM Kernel \eqref{eq:ssm-conv} can be quite computationally intensive. Ideally, when matrix $\bm{A}$ is diagonal, the computation simplifies to exponentiating the diagonal elements only. In line with the Diagonal Plus Low-Rank (\textbf{DPLR}) representation in S4 \cite{gu2021efficiently}, we introduce a low-rank correction term to the Hippo matrix in formula \eqref{eq:hippo-legs}\eqref{eq:hippo-legt} to obtain a skew-symmetric matrix $\bm{A}^{(N)}$ in formula \eqref{eq:hippo-legsd}\eqref{eq:hippo-legtd} that can be diagonalized.

\begin{lemma}\label{ssm_translate} 
For the $n$-dimensional SSM($\bm{A}$,$\bm{B}$,$\bm{C}$), any transformation defined by a nonsingular matrix $\bm{T}$ will generate a transformed SSM($\hat{\bm{A}}$,$\hat{\bm{B}}$,$\hat{\bm{C}}$):
\small
\[
\hat{\bm{A}}=\bm{T}^{-1} \bm{A} \bm{T}, \quad \hat{\bm{B}}=\bm{T}^{-1} \bm{B}, \quad \hat{\bm{C}}=\bm{C} \bm{T},
\]
\normalsize
where the dynamic matrix $\hat{\bm{A}}$ has the same characteristic polynomial and eigenvalues (same system dynamics) as $\bm{A}$, but its eigenvectors (dynamic coordinate space) are different. 
\end{lemma}
Lemma \ref{ssm_translate} defines a transformation rule for the SSM, which allows us to diagonalize matrix $\bm{A}$ into the complex domain using unitary matric $\bm{V}$, where $\bm{A} = \bm{V}^{-1}\bm{\Lambda}\bm{V}$ and $\bm{\Lambda}=\bm{A}^{(D)}=\operatorname*{eig}(\bm{A}^{(N)})$. $\bm{A}$ is unitarily equivalent to $\bm{A}^{(D)}$, but the complex domain typically has a stronger expressive power.



\begin{lemma} The TOSSM($\bm{ABC}$) is K-Lipschitz with $\operatorname*{eig}(\bm{A})\leq K^2$. 
\label{lipchiz}
\end{lemma}
S4D introduces a method known as the left half-plane control to ensure that the diagonal elements of matrix $\bm{A}^{(D)}$ remain negative. Lemma \ref{lipchiz} provides an alternative, more intuitive explanation, which protects the model from suffering recursive gradient explosion.
\begin{lemma}\cite{gu2022train}
  \label{prop:diag-not-ossm}
  There is no TOSSM with the diagonal state matrix $\bm{A} = \mathsf{diag}\{-1, -2, \dots\}$.
\end{lemma}
\begin{lemma}\cite{zhou2022film} Let $A$ be an unitary matrix and $\epsilon_t$ be $\sigma^2$-subgaussian random noise. We have: 
$$x_t= A^\theta x_{t-\theta}+\sum_{i=1}^{\theta-1} A^i b+\mathcal{O}(\sigma \sqrt{\theta}).$$
\label{noise}
\end{lemma}
\begin{corollary} The Diagonal SSM($\bm{A}$,$\bm{B}$) is a rough approximation of complete dynamics.
\end{corollary}
DSS \cite{gupta2022diagonal} proposed a pure complex diagonal SSM that discards low-rank correction terms. By Lemma \ref{nonlinear}, it can be considered as a rough approximation of the complete dynamics. In S4D, the special initialization method called S4D-real retains only the real diagonal part of \textbf{\textit{Hippo-LegS}} \eqref{eq:hippo-legs}. However, the explanation for this method is currently lacking, and it can be seen as a rougher approximation of robust dynamic representation. Notably, time series data is distinct from other sequential data due to the presence of prominent noise, resulting in cumulative errors in SSM systems (Lemma \ref{noise}). Surprisingly, Table \ref{ssms ablation} reveal that rough dynamic approximations, such as S4D-real, actually yield better performance, in contrast to findings in Long Range Arena (LRA) benchmarks \cite{tay2020long}.

\noindent \textbf{Discretization}.
For real-world time series data, it is necessary to convert continuous parameters ($\bm{\dt}\bm{A}\bm{B}$) to discrete parameters ($\bm{\overline{A}},\bm{\overline{B}}$). In this paper, a combination of the zero-order hold (Lemma \ref{general solution}) and the forward Euler \cite{gu2020hippo} method is employed to achieve a more concise discrete representation.
\begin{equation}
    (\textbf{ZOH}):\bm{\overline{A}} = \exp(\bm{\dt} \bm{A}) ,\quad \quad \quad (\textbf{\color{blue}Forward Euler}):\bm{\overline{B}} = \bm{\dt} \bm{B}.
\end{equation}

\section{\TS: State Space Model for Time Series}
We adhere to TSF standards by applying normalization and de-normalization techniques to the input and output, addressing the distribution shift issues \cite{liu2022non}. Furthermore, we employ linear embedding operations $\bm{P}$ on time series patches \cite{nie2022time} to obtain a vectorized representation of the input function $u(s)$ and use linear mapping $\bm{Q}$ back to the numerical observation domain $y(s)$. 

\begin{proposition} Patch Operation is a simplified piecewise polynomial approximation approach.
\end{proposition}

\subsection{Operator-like Architecture}
\begin{figure}[t]
        \centering
		\includegraphics[width=0.9\columnwidth]{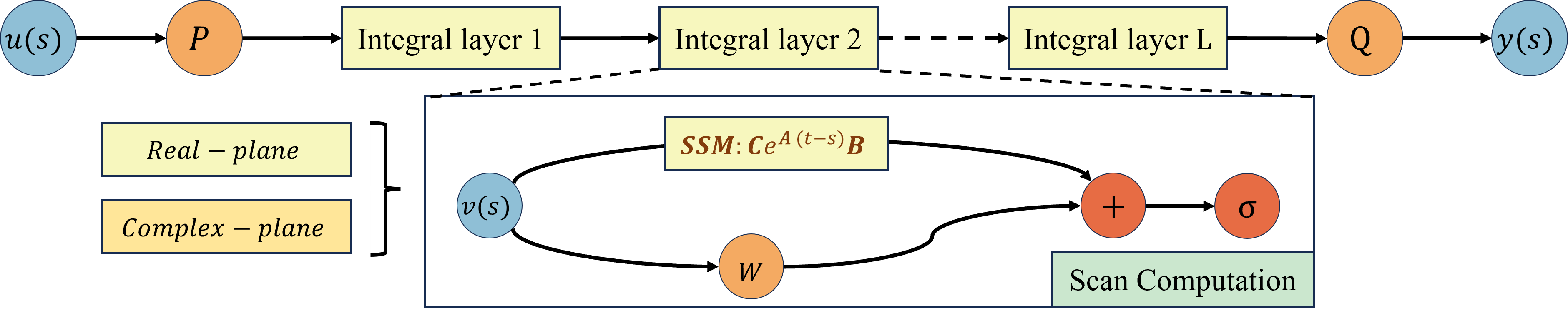}
		\caption{Architecture of \TS, $\bm{PQ}$ represent vectorized representation of a function.}
\end{figure}
In accordance with Definition \ref{def: dynamic operator}, \TS{} utilizes a standard neural operator structure, where each layer combines linear weight and SSM integral kernel. Additionally, the incorporation of an activation function introduces non-linearity in the model. Empirical evidence in Table \ref{tab:overviewall96} showcases the substantial superiority of this simple architecture over Mamba \cite{gu2023mamba} in the domain of TSF.
\begin{equation}
    y^{l} = \sigma(\bm{W}^{l}(u) + \bm{SSM}^{l}(u)).
\end{equation}
Table \ref{tab:various ssm} provides a comprehensive list of all SSM kernel variants introduced in this paper, and detailed explanations will be presented in Sections \ref{real-ssm} and Sections \ref{complex-ssm}.

\subsection{Real-plane SSM}\label{real-ssm}
The essence of SSM lies in computing the Krylov kernel of matrix $\bm{A}$, which determines the system's dynamic characteristics. Following Mamba's methodology, we parameterize $\bm{A}$ with a specific matrix and use linear layers to transform input $u$ into a time-varying ($\bm{BC}\bm{\dt}$) representation. By discretizing the system, we indirectly introduce time-varying properties to $\bm{A}$ by $\bm{\dt}\in\mathbbm{R}^{B\times L\times D}$. This allows $\bm{A}$ to still satisfy some properties of TOSSM in section \ref{background} within a time-varying SSM system. 
\begin{equation}
    \bm{A}\leftarrow \text{parameter},\quad\quad\bm{B},\bm{C}=\mathsf{Linear}_{\bm{BC}}(u),\quad\quad\bm{\dt}=\mathsf{softplus}(\mathsf{Linear}_{\bm{\dt}}(u)).
\end{equation}
\noindent \textbf{S4D-real.}~
we set matrix $\bm{A} = \mathsf{diag}\{-1, -2, \dots\}$ and use $\bm{A} =-\mathsf{exp}(\mathsf{log}(\mathsf{diag}\{1, 2, \dots\}))$ to stabilize the gradients \cite{wang2023stablessm, lezcano2019cheap}, which is the default setting in Mamba and our proposed Time-SSM.

\noindent \textbf{LegS/LegT/LegP.}~
$\bm{\dt}$ is fixed at input data frequency $1/L$, and the model parameters are initialized using a complete set of dynamics IOSSM ($\bm{ABC}$). The gradients from $\bm{ABC}$ are subsequently employed to adaptively update the state spectral space (Lemma \ref{space skewing}).

\begin{remark}
    we observe that retaining the complete dynamics of $\bm{A}\in\mathbbm{R}^{N\times N}$ in a time-varying ($\bm{BC}\bm{\dt}$) environment leads to gradient explosion. we adopt S4 to eliminate the time-varying factors.
\end{remark}

\noindent \textbf{Robust SSM.}~
Notably, the discretization process $\bm{\overline{B}} = \bm{\dt} \bm{B}$ also indirectly introduces time-varying factors from $\bm{\dt}$. To this end, we propose a robust version of the SSM where both ($\bm{AB}$) are specifically parameterized. Section \ref{ablation:b} presents a comprehensive performance analysis of this variant.

\subsection{Complex-plane SSM}\label{complex-ssm}
Following Lemma \ref{ssm_translate}, we conjugate diagonalize matrix $\bm{ABC}$ from real-plane to complex-plane. We initialize $\bm{V}\in\mathbbm{C}^{N\times N}$ with the diagonal unitary matrix of $\bm{A}$ and combine it with the time-varying $\bm{BC}$ generated by the linear layer, which allows SSM($\bm{ABC}$) to possess compatible dynamics. 
\begin{equation}
    \bm{B}=\bm{V}^{T}\mathsf{Linear}_{\bm{B}}(u)\quad(\text{unitary matrix}: \bm{V}^{T}=\bm{V}^{-1}),\quad\quad\quad\bm{C}=\bm{V}\mathsf{Linear}_{\bm{C}}(u).
\end{equation}
In the Experiment Section, we observe that this operation is necessary. Directly initializing the imaginary part \( \Im(\mathsf{Linear}_{\bm{BC}}(u)) =0\) leads to significant performance degradation. Time-SSM provides three complex-plane SSM kernels in this paper: \textbf{LegS/LegT/LegP-complex}.
\begin{table}[t]
\small
        \centering
	\caption{Parameterization choices for structured SSM kernels}
        \setlength{\tabcolsep}{7pt}
        \begin{adjustbox}{width=1\columnwidth, center}
        \resizebox{1\columnwidth}{!}{
        \begin{tabular}{c|cccc}
\toprule
Method&{{Structure of $\bm {A}$}}&{Structure of $\bm {B}$}&{Structure of $\bm {C}$} &{Time-varing factor}\\
\midrule
\textbf{S4D-real}& $\mathbbm{R}(D, N) \leftarrow$ diagonal & $\mathbbm{R}(B, L, N): \mathsf{Linear}_{B}(x)$ & $\mathbbm{R}(B, L, N): \mathsf{Linear}_{C}(x)$ &$\Delta, \bm {B}, \bm {C}$\\
\textbf{LegS/LegT/LegP-complex} & $\mathbbm{C}(D, N) \leftarrow$ diagonal &$\mathbbm{C}(B, L, N): \bm {V}^{-1}\mathsf{Linear}_{B}(x)$& $\mathbbm{C}(B, L, N): \bm {V}\mathsf{Linear}_{C}(x)$ &$\Delta, \bm {B}, \bm {C}$\\
\textbf{LegS/LegT}& $\mathbbm{R}(D, N, N) \leftarrow$ NPLR &$\mathbbm{R}(D, N)\leftarrow $parameter& $\mathbbm{R}(D, N)\leftarrow $parameter& None\\
\textbf{Robust version}& any above & $\mathbbm{R}(N)\leftarrow $parameter & $\mathbbm{R}(B, L, N): \mathsf{Linear}_{C}(x)$ &$\Delta, \bm {C}$\\
\bottomrule
\end{tabular}
}
\end{adjustbox}
\label{tab:various ssm}
\end{table}

\section{Experiment}
In this section, we commence by conducting a comprehensive performance comparison of Time-SSM against other state-of-the-art models, followed by ablation experiments pertaining to initialization methods and model architectures. Furthermore, considering the characteristics of TSF tasks, we perform autoregressive prediction experiments, complexity analysis, and representation ability analysis. Except for {LegS and LegT}, all other variants of SSM kernel can be implemented using the diagonal scanning algorithm like S5~\cite{smith2022simplified}. For detailed information regarding baseline models, dataset descriptions, experimental settings, and hyper-parameter analysis, please refer to Appendix \ref{add exp}.

\subsection{Overall Performance}
\begin{table*}[t]
\centering
\caption{Average results of long-term forecasting with an input length of 96 and prediction horizons of \{96, 192, 336, 720\}. The best performance is in {\color{red}Red}, and the second best is in {\color{blue}Blue}. Different colors represent SSM-based and other neural-network-based TSF models. Full results are in Appendix \ref{full result}.}
\setlength{\tabcolsep}{6pt}
\begin{adjustbox}{width=1\columnwidth, center}
\resizebox{1\columnwidth}{!}{
\begin{tabular}{c|cc|cc|cc|cc|cc|cc|cc|cc}
\toprule[1.5pt]
\multicolumn{1}{c}{\multirow{2}{*}{\large{Model}}}&\multicolumn{2}{c}{\cellcolor{mycolor4}{\TS}}&\multicolumn{2}{c}{\cellcolor{mycolor4}Mamba4TS}&\multicolumn{2}{c}{\cellcolor{mycolor4}S-Mamba}&\multicolumn{2}{c}{\cellcolor{mycolor4}RWKV-TS} &\multicolumn{2}{c}{\cellcolor{mycolor4}Koopa}&\multicolumn{2}{c}{\cellcolor{mycolor3}InvTrm}&\multicolumn{2}{c}{\cellcolor{mycolor3}PatchTST}&\multicolumn{2}{c}{\cellcolor{mycolor3}DLinear}\\
\multicolumn{1}{c}{} & \multicolumn{2}{c}{\cellcolor{mycolor4}{(Ours)}} & \multicolumn{2}{c}{\cellcolor{mycolor4}(Temporal Emb.)}&\multicolumn{2}{c}{\cellcolor{mycolor4}(Channel Emb.\cite{wang2024mamba})} & \multicolumn{2}{c}{\cellcolor{mycolor4}\cite{hou2024rwkv}} & \multicolumn{2}{c}{\cellcolor{mycolor4}\cite{liu2024koopa}} & \multicolumn{2}{c}{\cellcolor{mycolor3}\cite{liu2023itransformer}} & \multicolumn{2}{c}{\cellcolor{mycolor3}\cite{nie2022time}} & \multicolumn{2}{c}{\cellcolor{mycolor3}\cite{zeng2023transformers}} \\
\cmidrule(l){1-1}\cmidrule(l){2-3}\cmidrule(l){4-5}\cmidrule(l){6-7}\cmidrule(l){8-9}\cmidrule(l){10-11}\cmidrule(l){12-13}\cmidrule(l){14-15}\cmidrule(l){16-17}
\multicolumn{1}{c}{Metric}&MSE & MAE & MSE &MAE & MSE &MAE & MSE & MAE & MSE & MAE & MSE & MAE & MSE & MAE & MSE & MAE \\
\midrule
\midrule
ETTh1
& \textbf{\color{red}0.425} & \textbf{\color{red}0.426} & 0.444 & 0.438 & 0.459 & 0.453 & 0.454 & 0.446 & 0.450 & 0.443 & 0.463 & 0.454 & \textbf{\color{blue}0.434} & \textbf{\color{blue}0.435} & 0.462 & 0.458 \\
\midrule
ETTh2
& \textbf{\color{red}0.374} & \textbf{\color{red}0.399} & 0.386 & 0.410 & 0.381 & 0.407 &\textbf{\color{blue}0.375} & \textbf{\color{blue}0.402} & 0.397 & 0.417 & 0.383 & 0.407 & 0.380 & 0.406 & 0.564 & 0.520 \\
\midrule
ETTm1
& \textbf{\color{red}0.386} & \textbf{\color{red}0.396} & 0.396 & 0.406 & 0.399 & 0.407 & \textbf{\color{blue}0.391} & 0.403  & 0.395 & 0.403 & 0.407 & 0.412 & 0.403 & \textbf{\color{blue}0.398} & 0.403 & 0.406 \\
\midrule
ETTm2
& \textbf{\color{blue}0.283} & \textbf{\color{blue}0.328} & 0.299 & 0.343 & 0.289 & 0.333 & 0.285 & 0.330 & \textbf{\color{red}0.281} & \textbf{\color{red}0.326} & 0.291 & 0.335 & \textbf{\color{blue}0.283} & 0.329 & 0.345 & 0.396 \\
\midrule
Exchange
& \textbf{\color{blue}0.352} & \textbf{\color{red}0.398} & 0.364 & \textbf{\color{blue}0.405} & 0.364 & 0.407 & 0.406 & 0.439 &  0.390 & 0.424 & 0.366 & 0.416 & 0.383 & 0.416 & \textbf{\color{red}0.346} & 0.416 \\
\midrule
Crypto
& \textbf{\color{blue}0.192} & \textbf{\color{blue}0.160} & 0.193 & 0.162 & 0.198 & 0.163 & \textbf{\color{red}0.190} & \textbf{\color{red}0.159} & 0.199 & 0.165 & 0.196 & 0.164 & \textbf{\color{blue}0.192} & 0.161 & 0.201 & 0.176\\
\midrule
Air-convection
& \textbf{\color{red}0.459} &\textbf{\color{red}0.332} & 0.471 & 0.343 & 0.484 & 0.352 &0.464 & 0.336 & \textbf{\color{blue}0.463} & 0.337 & 0.493 & 0.363 & 0.483 & 0.354 & \textbf{\color{red}0.459} & \textbf{\color{blue}0.341} \\
\midrule
Weather
& \textbf{\color{blue}0.252} &\textbf{\color{blue}0.276} & 0.258 & 0.280 & \textbf{\color{blue}0.252} & 0.277 &0.256 & 0.280 & \textbf{\color{red}0.247} & \textbf{\color{red}0.273} & 0.260 & 0.280 & 0.258 & 0.280 & 0.267 & 0.319 \\
\bottomrule[1.5pt]
\end{tabular}}
\label{tab:overviewall96}
\end{adjustbox}
\end{table*}

As depicted in Table \ref{tab:overviewall96}, we can observe that (a) Time-SSM (S4D-real) achieved the best overall performance, followed by Koopa and RWKV-TS. This highlights the necessity of modeling temporal dynamics in the field of TSF. (b) Linear model DLinear surprisingly performs well on the Exchange and Air-convection datasets, where the former is believed to be influenced by external factors such as the economy, and the latter is considered a chaotic time series \cite{hu2024attractor}. This observation further highlights the importance of robust modeling in TSF tasks. (c) Time-SSM eliminates redundant modules from Mamba while exhibiting significant performance enhancements compared to Mamba4TS. We envision Time-SSM as a novel foundational model that can drive further advancements in future research. (d) In datasets with a relatively small number of variables, both SSM-based (Mamba4TS \textit{vs.} S-Mamba) and Transformer-based models (PatchTST \textit{vs.} InvTrm) still consider temporal embedding as the optimal choice. Variable embedding may lead to improved results in datasets with a larger number of variables, 
which is evident in datasets like Traffic \cite{wu2021autoformer}, where the number of variables (862) greatly exceeds the time size (96). In Section \ref{variable kernel}, we present a lightweight solution to balance the temporal dependency and variable dependency, which is not the focus (\textbf{\textit{temporal dynamic}}) of this paper.
\subsection{Ablation Study}\label{exp:ablation}
\begin{table*}[t!]
\centering
\caption{Ablation study of $\mathbf{ABC}$ parameterization. {\color{red}Red}/{\color{blue}Blue} denotes the best/second performance.}
\label{ssms ablation}
\setlength{\tabcolsep}{10pt}
\begin{adjustbox}{width=1\columnwidth, center}
\resizebox{1\columnwidth}{!}{
\begin{tabular}{p{0.3cm}|c|cc|cc|cc|cc|cc|cc|cc}
\toprule[1.5pt]
\multicolumn{2}{c}{\multirow{1}{*}{{SSMs}}}&\multicolumn{2}{c}{\cellcolor{mycolor4}{S4D-real}}&\multicolumn{2}{c}{\cellcolor{mycolor4}LegS-complex}&\multicolumn{2}{c}{\cellcolor{mycolor4}LegT-complex} &\multicolumn{2}{c}{\cellcolor{mycolor4}LegP-complex} &\multicolumn{2}{c}{\cellcolor{mycolor4}LegS(RNN)}&\multicolumn{2}{c}{\cellcolor{mycolor4}LegT(RNN)}&\multicolumn{2}{c}{\cellcolor{mycolor4}Full-Select}\\
\cmidrule(l){1-2}\cmidrule(l){3-4}\cmidrule(l){5-6}\cmidrule(l){7-8}\cmidrule(l){9-10}\cmidrule(l){11-12}\cmidrule(l){13-14}\cmidrule(l){15-16}
\multicolumn{2}{c}{Metric}&MSE & MAE & MSE &MAE & MSE &MAE & MSE & MAE & MSE & MAE & MSE & MAE & MSE & MAE \\
\midrule
\midrule
\multirow{5}{*}{\begin{sideways}ETTh1\end{sideways}} 
& 96  & \textbf{\color{red}0.377} & \textbf{\color{red}0.394} & 0.382 & 0.398 & \textbf{\color{blue}0.379} & \textbf{\color{blue}0.395} & 0.381 & 0.397 & 0.391 & 0.405 & 0.388 &0.403& 0.382 & 0.399 \\
& 192 & \textbf{\color{red}0.423} & \textbf{\color{red}0.424} & 0.429 & 0.427 & 0.426 & \textbf{\color{blue}0.426} & \textbf{\color{blue}0.424} & 0.427 & 0.438 & 0.436 & 0.434 &0.433& 0.435 & 0.429 \\
& 336 & \textbf{\color{blue}0.466} & \textbf{\color{blue}0.437} & 0.471 & 0.450 & 0.470 & 0.451 & \textbf{\color{red}0.464} & \textbf{\color{red}0.435} & 0.475 & 0.451 & 0.470 &0.445& 0.482 & 0.452 \\
& 720 & \textbf{\color{red}0.452} & \textbf{\color{red}0.448} & 0.474 & 0.471 & 0.469 & 0.467 & \textbf{\color{blue}0.460} & \textbf{\color{blue}0.455} & 0.488 & 0.483 & 0.484 &0.477& 0.483 & 0.476 \\
& AVG & \textbf{\color{red}0.430} & \textbf{\color{red}0.426} & 0.439 & 0.437 & 0.436 & 0.435 & \textbf{\color{blue}0.432} & \textbf{\color{blue}0.429} & 0.448 & 0.444 & 0.444 & 0.440 & 0.446 & 0.439 \\
\midrule
\multirow{5}{*}{\begin{sideways}ETTm2\end{sideways}} 
& 96  & \textbf{\color{blue}0.176} & \textbf{\color{blue}0.260} & 0.179 & 0.265 & \textbf{\color{blue}0.176} & 0.263 & \textbf{\color{red}0.174} & \textbf{\color{red}0.259} & 0.188 & 0.274 & 0.184 & 0.273 & 0.180 & 0.265 \\
& 192 & 0.246 & \textbf{\color{red}0.305} & 0.246 & 0.308 & \textbf{\color{red}0.242} & \textbf{\color{blue}0.307} & \textbf{\color{blue}0.244} & 0.308 & 0.264 & 0.321 & 0.262 & 0.318 & 0.247 & 0.309 \\
& 336 & 0.305 & 0.344 & 0.310 & 0.348 & \textbf{\color{blue}0.304} & \textbf{\color{blue}0.343} & \textbf{\color{red}0.301} & \textbf{\color{red}0.339} & 0.321 & 0.358 & 0.317 & 0.355 & 0.312 & 0.351 \\
& 720 & \textbf{\color{red}0.406} & \textbf{\color{red}0.405} & 0.411 & 0.410 & \textbf{\color{blue}0.409} & \textbf{\color{blue}0.407} & 0.410 &0.408 & 0.431 & 0.436 & 0.430 & 0.428 & 0.412 & 0.410 \\
& AVG & \textbf{\color{blue}0.283} & \textbf{\color{red}0.329} & 0.287 & 0.333 & \textbf{\color{blue}0.283} & \textbf{\color{blue}0.330} & \textbf{\color{red}0.282} & \textbf{\color{red}0.329} & 0.301 & 0.347 & 0.298 & 0.343 & 0.288 & 0.334 \\
\midrule
\multirow{5}{*}{\begin{sideways}Weather\end{sideways}} 
& 96  & \textbf{\color{red}0.167} & \textbf{\color{red}0.212} & 0.175 & 0.218 & 0.178 & 0.221 & \textbf{\color{blue}0.170} & \textbf{\color{blue}0.215} & 0.189 &0.228& 0.184 &0.222& 0.185 & 0.225 \\
& 192 & \textbf{\color{blue}0.217} & \textbf{\color{blue}0.255} & 0.222 & \textbf{\color{red}0.253} & 0.226 & 0.259 & \textbf{\color{red}0.216} & 0.257 & 0.238 &0.269& 0.236 &0.263& 0.230 & 0.262 \\
& 336 & \textbf{\color{blue}0.274} & \textbf{\color{blue}0.294} & 0.285 & 0.303 & 0.282 & 0.297 & \textbf{\color{red}0.270} & \textbf{\color{red}0.291} & 0.291 &0.310& 0.294 &0.311& 0.285 & 0.302 \\
& 720 & \textbf{\color{red}0.351} & \textbf{\color{red}0.345} & 0.354 & \textbf{\color{blue}0.346} & 0.355 & \textbf{\color{red}0.345} & \textbf{\color{blue}0.353} & 0.348 & 0.366 &0.358& 0.362 &0.355& 0.360 & 0.351 \\
& AVG & \textbf{\color{red}0.252} & \textbf{\color{red}0.277} & \textbf{\color{blue}0.259} & 0.280 & 0.260 & 0.281 & \textbf{\color{red}0.252} & \textbf{\color{blue}0.278} & 0.271 & 0.291 & 0.269 & 0.288 & 0.265 & 0.285 \\
\bottomrule[1.5pt]
\end{tabular}}
\end{adjustbox}
\end{table*}

Next, we will explore the practical performance of different SSM kernels and model architectures.
\begin{table*}[t!]
\centering
\caption{Architecture ablation with various prediction horizons. \textit{w/o} denotes without the corresponding module. B: Taime-varying $\bm{B}$ matrix; V: diagonal unitary matrix; W: Linear weights in each layer; VK: Variable kernel. {\color{red}Red} denotes improved performance, while {\color{blue}{Blue}} denotes the declined performance.}
\label{architecture ablation}
\begin{adjustbox}{width=1\columnwidth, center}
\begin{minipage}{.707\linewidth}%
\resizebox{1\columnwidth}{!}{
\setlength{\tabcolsep}{6pt}
\begin{tabular}{p{0.3cm}|c||cc|cc|cc||cc|cc|cc}
\toprule[1.5pt]
\multicolumn{2}{c||}{\multirow{1}{*}{{Model}}}&\multicolumn{2}{c}{\cellcolor{mycolor4}{S4D-real}}&\multicolumn{2}{c}{\cellcolor{mycolor4}\textit{w/o} B}&\multicolumn{2}{c||}{\cellcolor{mycolor4}\textit{w/o} W} &\multicolumn{2}{c}{\cellcolor{mycolor4}LegS-complex} &\multicolumn{2}{c}{\cellcolor{mycolor4}\textit{w/o} B}&\multicolumn{2}{c}{\cellcolor{mycolor4}\textit{w/o} V}\\
\cmidrule(l){1-2}\cmidrule(l){3-4}\cmidrule(l){5-6}\cmidrule(l){7-8}\cmidrule(l){9-10}\cmidrule(l){11-12}\cmidrule(l){13-14}
\multicolumn{2}{c||}{Metric}&MSE & MAE & MSE &MAE & MSE &MAE & MSE & MAE & MSE & MAE & MSE & MAE\\
\midrule
\midrule
\multirow{5}{*}{\begin{sideways}ETTh1\end{sideways}} 
& 96  & 0.377 &0.394 & \textbf{\color{blue}0.382} &\textbf{\color{blue}0.401} & \textbf{\color{blue}0.380} &\textbf{\color{blue}0.400} & 0.382 & 0.398 & \textbf{\color{blue}0.385} & \textbf{\color{blue}0.403} & \textbf{\color{blue}0.385} & \textbf{\color{red}0.396} \\
& 192 & 0.423 & 0.424 & \textbf{\color{blue}0.438} & \textbf{\color{blue}0.436} & \textbf{\color{blue}0.432} & \textbf{\color{blue}0.429} & 0.429 & 0.427 & \textbf{\color{blue}0.445} & \textbf{\color{blue}0.435} & \textbf{\color{blue}0.435} & \textbf{\color{blue}0.432} \\
& 336 & 0.466 & 0.437 & \textbf{\color{blue}0.483} & \textbf{\color{blue}0.452} & \textbf{\color{blue}0.476} & \textbf{\color{blue}0.452} & 0.471 & 0.450 & \textbf{\color{blue}0.489} & \textbf{\color{blue}0.456} & \textbf{\color{blue}0.478} & \textbf{\color{blue}0.453} \\
& 720 & 0.452 &0.448 & \textbf{\color{blue}0.510} & \textbf{\color{blue}0.489} & \textbf{\color{blue}0.497} & \textbf{\color{blue}0.471} & 0.474 & 0.471 & \textbf{\color{blue}0.524} & \textbf{\color{blue}0.495} & \textbf{\color{blue}0.478} & \textbf{\color{red}0.465} \\
& AVG & 0.425 & 0.426 & \textbf{\color{blue}0.453} & \textbf{\color{blue}0.445} & \textbf{\color{blue}0.446} & \textbf{\color{blue}0.438} & 0.439 & 0.437 & \textbf{\color{blue}0.461} & \textbf{\color{blue}0.447} & \textbf{\color{blue}0.444} & 0.437 \\
\midrule
\multirow{5}{*}{\begin{sideways}ETTh2\end{sideways}} 
& 96  & 0.290 &0.341 & \textbf{\color{blue}0.298} & \textbf{\color{blue}0.349} & \textbf{\color{red}0.288} & \textbf{\color{red}0.339} & 0.298 & 0.347 & \textbf{\color{blue}0.308} & \textbf{\color{blue}0.358} & \textbf{\color{blue}0.305} & \textbf{\color{blue}0.357} \\
& 192 & 0.368 &0.387 & \textbf{\color{blue}0.382} & \textbf{\color{blue}0.403} & \textbf{\color{blue}0.371} & \textbf{\color{blue}0.390} & 0.377 & 0.395 & \textbf{\color{blue}0.384} & \textbf{\color{blue}0.405} & 0.377 & \textbf{\color{blue}0.400} \\
& 336 & 0.416 &0.430& \textbf{\color{blue}0.420} & \textbf{\color{red}0.427} & \textbf{\color{red}0.412} & \textbf{\color{red}0.425} & 0.421 & 0.434 & \textbf{\color{blue}0.433} & \textbf{\color{blue}0.442} & \textbf{\color{blue}0.423} & 0.434 \\
& 720 & 0.424 &0.439 & \textbf{\color{red}0.422} & \textbf{\color{blue}0.441} & \textbf{\color{red}0.418} & \textbf{\color{red}0.438} & 0.427 & 0.441 & \textbf{\color{blue}0.434} & \textbf{\color{blue}0.452} & \textbf{\color{red}0.426} & \textbf{\color{blue}0.447} \\
& AVG & 0.374 &0.399 & \textbf{\color{blue}0.381 }& \textbf{\color{blue}0.405} & \textbf{\color{red}0.372} & \textbf{\color{red}0.398} & 0.381 & 0.404 & \textbf{\color{blue}0.390} & \textbf{\color{blue}0.414} & \textbf{\color{blue}0.383} & \textbf{\color{blue}0.410} \\
\midrule
\multirow{5}{*}{\begin{sideways}Air\end{sideways}} 
& 96  & 0.290 &0.341 & \textbf{\color{blue}0.292} & \textbf{\color{blue}0.343} & \textbf{\color{red}0.286} & \textbf{\color{red}0.338} & 0.296 & 0.345 & \textbf{\color{blue}0.299} & \textbf{\color{blue}0.351} & \textbf{\color{blue}0.299} & \textbf{\color{blue}0.352} \\
& 192 & 0.368 &0.387 & \textbf{\color{red}0.366} & \textbf{\color{red}0.386} & \textbf{\color{blue}0.372} & \textbf{\color{blue}0.391} &0.371 & 0.388 & \textbf{\color{blue}0.377} & \textbf{\color{blue}0.394} & \textbf{\color{blue}0.376} & \textbf{\color{blue}0.392} \\
& 336 & 0.416 &0.430& \textbf{\color{red}0.411} & \textbf{\color{red}0.423} & \textbf{\color{blue}0.419} & \textbf{\color{blue}0.434} & 0.421 & 0.436 & \textbf{\color{blue}0.424} & 0.436 & \textbf{\color{blue}0.430} & \textbf{\color{blue}0.441} \\
& 720 & 0.424 &0.439 & \textbf{\color{red}0.419} & \textbf{\color{red}0.434} & \textbf{\color{blue}0.432} & \textbf{\color{blue}0.447} &0.430 & 0.444 & \textbf{\color{red}0.428} & \textbf{\color{red}0.442} & \textbf{\color{red}0.428} & \textbf{\color{blue}0.449} \\
& AVG & 0.374 &0.399 & \textbf{\color{red}0.372} & \textbf{\color{red}0.396} & \textbf{\color{blue}0.377} & \textbf{\color{blue}0.403} & 0.380 & 0.403 & \textbf{\color{blue}0.382} & \textbf{\color{blue}0.406} & \textbf{\color{blue}0.383} & \textbf{\color{blue}0.409} \\
\bottomrule[1.5pt]
\end{tabular}}
\end{minipage}
\begin{minipage}{.293\linewidth}%
\resizebox{1\columnwidth}{!}{
\setlength{\tabcolsep}{6pt}
\begin{tabular}{p{0.3cm}|c|cc|cc}
\toprule[1.5pt]
\multicolumn{2}{c}{\multirow{1}{*}{{Model}}}&\multicolumn{2}{c}{\cellcolor{mycolor4}{S4D-real}}&\multicolumn{2}{c}{\cellcolor{mycolor4}\textit{with} VK}\\
\cmidrule(l){1-2}\cmidrule(l){3-4}\cmidrule(l){5-6}
\multicolumn{2}{c}{Metric}&MSE & MAE & MSE &MAE \\
\midrule
\midrule
\multirow{5}{*}{\begin{sideways}ETTh2\end{sideways}} 
& 96  & {0.290} &{0.341} & \textbf{\color{blue}0.305} & \textbf{\color{blue}0.351} \\
& 192 & {0.368} &0.387 & \textbf{\color{blue}0.372} & \textbf{\color{red}0.386} \\
& 336 & 0.416 &0.430& \textbf{\color{red}0.399} & \textbf{\color{red}0.416} \\
& 720 & 0.424 &0.439 & \textbf{\color{red}0.418} & \textbf{\color{red}0.429} \\
& AVG & 0.374 &0.399 & \textbf{\color{red}0.373} & \textbf{\color{red}0.395 }\\
\midrule
\multirow{5}{*}{\begin{sideways}ETTm2\end{sideways}} 
& 96  & {0.176} & 0.260 & 0.176 & \textbf{\color{red}0.258}\\
& 192 & {0.246} & {0.305} & \textbf{\color{blue}0.247} & \textbf{\color{blue}0.309} \\
& 336 & 0.305 &0.344 & \textbf{\color{red}0.301} & \textbf{\color{red}0.334} \\
& 720 & 0.406 & 0.405 & \textbf{\color{red}0.392} & \textbf{\color{red}0.397} \\
& AVG & 0.283 & 0.328 & \textbf{\color{red}0.279} & \textbf{\color{red}0.325}\\
\midrule
\multirow{5}{*}{\begin{sideways}Weather\end{sideways}} 
& 96  & 0.167 & 0.212 & \textbf{\color{red}0.154} & \textbf{\color{red}0.202} \\
& 192 & 0.217 & 0.255 & \textbf{\color{red}0.204} & \textbf{\color{red}0.242} \\
& 336 & 0.274 & 0.294 & \textbf{\color{red}0.262} & \textbf{\color{red}0.288} \\
& 720 & 0.351 & 0.345 & \textbf{\color{red}0.339} & \textbf{\color{red}0.341} \\
& AVG & 0.252 & 0.276 & \textbf{\color{red}0.239} & \textbf{\color{red}0.268} \\
\bottomrule[1.5pt]
\end{tabular}}
\end{minipage}
\end{adjustbox}
\end{table*}

\noindent \textbf{Various SSM Kernels.}~
Table \ref{ssms ablation} demonstrates the results of Time-SSM using all SSM kernels from Table \ref{tab:various ssm}. Overall, S4D-real and LegP-complex parameterizations achieve the best performance, and S4D-real initialization remains the simplest and most effective choice for TSF tasks. Removing time-varying factors in LegS and LegT leads to poor performance, emphasizing their importance as attention or gating mechanisms. Surprisingly, generating $\bm{A}$ by linear layer (full select) results in decreased performance, highlighting the necessity of selecting specific dynamic matrices $\bm{A}$.

Interestingly, LegS-complex generally outperforms other variants in LRA benchmark tasks (text, images, logical reasoning), with the support of exponential decay memory. However, in TSF tasks, the opposite phenomenon occurs, which suggests that \textbf{\textit{in TSF settings, where the lookback window is limited, robust approximation and finite window approximation yield better performance.}}

\noindent \textbf{Time-varying $\bm {B}$.}~
\label{ablation:b}
As shown in Table \ref{architecture ablation} (left), removing time-varying factors from matrix B (initialized by the Hippo matrix) generally leads to decreased performance. However, in datasets with prominent nonlinear features like Air-convection, especially in long-term prediction windows, it produces more robust predictions. This effect is somewhat weakened in the complex-plane version of Time-SSM.

\noindent \textbf{Unitary matrix $\bm {V}$.}~
Removing the unitary matrix $\bm {V}$ from the $\bm {BC}$ matrix consistently leads to performance degradation. We view \textbf{\textit{$\bm {V}$ as an inductive bias that imparts $\bm {BC}$ with complex-plane dynamics that are compatible with $\bm {A}$ matrix, which is essential for temporal dynamical modeling.}}

\noindent \textbf{Linear Weights in Operator.}~
Preserving the linear weights in each layer generally yields better results. We believe that the performance decline in the ETTh2 dataset may be attributed to overfitting by non-stationary distributions, which remains consistent in both real-plane and complex-plane.

\label{variable kernel}
\noindent \textbf{Multivariate SSM Kernel.}~
In the era of LLM, we believe that excessive focus on modeling variable relationships within individual datasets becomes less meaningful. Models like UniTime \cite{liu2023unitime}, Timer~\cite{liu2024timer}, and MOIRAI \cite{woo2024unified} have already begun exploring unified time series models. Therefore, in this paper, \textbf{\textit{we prioritize modeling temporal dependencies and aim to enhance model performance on small-scale datasets by leveraging inter-variable relationships at minimal cost.}} Two principles are proposed: (i) As for a sequence model, it is desirable to achieve a complexity of $\mathcal{O}(L)$. (ii) Given that the number of variables in different datasets may vary, it is preferable to impose an upper limit on the parameter count of variables rather than allowing unrestricted growth. To achieve this, we employ Fourier transformation to capture dominant $k$ modes along variable dimensions and utilize complex-domain linear weights to represent the variable integral kernel: \(\hat{y}=\mathsf{IFFT}(\bm{W}\cdot \mathsf{topk}(\mathsf{FFT}(y)))\). As shown in Table \ref{architecture ablation} (right), this simple and efficient module leads to a significant performance improvement, particularly on datasets with more variables.

\subsection{Further Analysis}\label{further exp}
\noindent \textbf{Auto Regression.}~
SpaceTime \cite{zhang2023effectively} proposed a closed-loop control prediction method using input time series $u_{[1:L]}$ to construct a dynamical system and generate predictions $y_{[L+1:L+H]}$ through autoregressive modeling. However, this approach is not suitable for time-varying systems with changing matrices ($\bm{ABC}$). To address this, we introduced an implicit autoregressive method with zero padding: $y\in\mathbbm{R/C}^{L+H}$=~{\TS}~$(u\in\mathbbm{R/C}^{L+O})$. The final prediction is obtained by truncating the last $H$ steps of $y$. Experimental results (Table \ref{tab:ar}) show decreased performance for all initialization methods under this configuration, indicating redundancy of the additional autoregressive generation. Surprisingly, in longer sequences, the LegS-based variant outperforms LegT, contrary to Table \ref{ssms ablation} findings, supporting the original benefit of LegS for long-term sequence modeling. Moreover, selecting the $\bm{BC}$ matrix from the last time step for autoregressive generation leads to gradient explosions in time-varying environments due to the lack of dynamic constraints.
\begin{table*}[t!]
\centering
\caption{Forecasting results in {\color{blue}AR} process with various prediction horizons. $\bm{\infty}$ means gradient explosions.}
\setlength{\tabcolsep}{10pt}
\begin{adjustbox}{width=1\columnwidth, center}
\resizebox{0.8\columnwidth}{!}{
\begin{tabular}{p{0.3cm}|c|cc|cc|cc|cc|cc|cc}
\toprule[1.5pt]
\multicolumn{2}{c}{\multirow{1}{*}{{SSMs}}}&\multicolumn{2}{c}{\cellcolor{mycolor4}{S4D-real}}&\multicolumn{2}{c}{\cellcolor{mycolor4}LegS-complex}&\multicolumn{2}{c}{\cellcolor{mycolor4}LegT-complex} &\multicolumn{2}{c}{\cellcolor{mycolor4}LegP-complex} &\multicolumn{2}{c}{\cellcolor{mycolor4}Full-Select}&\multicolumn{2}{c}{\cellcolor{mycolor4}(S4D-real) Recurrent}\\
\cmidrule(l){1-2}\cmidrule(l){3-4}\cmidrule(l){5-6}\cmidrule(l){7-8}\cmidrule(l){9-10}\cmidrule(l){11-12}\cmidrule(l){13-14}
\multicolumn{2}{c}{Metric}&MSE & MAE & MSE &MAE & MSE &MAE & MSE & MAE & MSE & MAE & MSE & MAE\\
\midrule
\midrule
\multirow{5}{*}{\begin{sideways}ETTm2\end{sideways}} 
& 96  &\textbf{ \color{blue}0.184} & \textbf{\color{blue}0.272} & 0.189 & 0.278 & 0.188 & 0.275 & \textbf{\color{red}0.181}  & \textbf{\color{red}0.270} & 0.192 & 0.276 & 0.198 & 0.296\\
& 192 & \textbf{\color{red}0.251} & \textbf{\color{blue}0.313} & 0.255 & 0.319 & \textbf{\color{blue}0.254} & 0.316 & \textbf{\color{red}0.251}  & \textbf{\color{red}0.312} & 0.256 & 0.316 & 0.264 & 0.327\\
& 336 & \textbf{\color{blue}0.320} & 0.360 & \textbf{\color{blue}0.320} & \textbf{\color{blue}0.359} & 0.324 & 0.362 & 0.324  & 0.360 & \textbf{\color{red}0.317} & \textbf{\color{red}0.354} & 0.337 & 0.368 \\
& 720 & 0.420 & 0.417 & \textbf{\color{red}0.417} & \textbf{\color{red}0.413} & \textbf{\color{blue}0.418} & \textbf{\color{blue}0.415} & 0.420  & 0.417 & 0.423 & \textbf{\color{red}0.413} & $\bm{\infty}$  &$\bm{\infty}$\\
& AVG & \textbf{\color{red}0.294} & \textbf{\color{blue}0.341} & \textbf{\color{blue}0.295} & 0.342 & 0.296 & 0.342 & \textbf{\color{red}0.294} & \textbf{\color{red}0.340} & 0.297 & \textbf{\color{red}0.340} & $\bm{\infty}$  &$\bm{\infty}$\\
\midrule
\multirow{5}{*}{\begin{sideways}Weather\end{sideways}} 
& 96  & \textbf{\color{red}0.182} & \textbf{\color{red}0.223} & 0.190 & 0.229 & 0.194 & 0.234 & \textbf{\color{blue}0.186}  &0.228& 0.190 & 0.229 & \textbf{\color{blue}0.186}  & \textbf{\color{blue}0.227} \\
& 192 & \textbf{\color{red}0.227} & \textbf{\color{red}0.260} & \textbf{\color{blue}0.234} & \textbf{\color{blue}0.266} & 0.242 & 0.273 & 0.236  & 0.269& 0.239 & 0.269 & 0.236  & 0.270 \\
& 336 & \textbf{\color{red}0.282} & \textbf{\color{blue}0.300} & 0.289 & 0.304 & 0.298 & 0.312 & \textbf{\color{blue}0.284}  & \textbf{\color{red}0.299} & 0.295 & 0.308 & 0.315& 0.330\\
& 720 & \textbf{\color{red}0.361} & \textbf{\color{blue}0.351} & \textbf{\color{blue}0.362} & \textbf{\color{blue}0.351} & 0.375 & 0.362 & \textbf{\color{red}0.361}  & \textbf{\color{red}0.350} & 0.366 & 0.354 & $\bm{\infty}$  &$\bm{\infty}$\\
& AVG & \textbf{\color{red}0.263} & \textbf{\color{red}0.284} & 0.269 & 0.288 & 0.277 & 0.295 & \textbf{\color{blue}0.267} & \textbf{\color{blue}0.287} & 0.273 & 0.290 & $\bm{\infty}$  &$\bm{\infty}$\\
\bottomrule[1.5pt]
\end{tabular}}
\label{tab:ar}
\end{adjustbox}
\end{table*}
\begin{figure}[t]
        \centering
\begin{adjustbox}{width=1\columnwidth, center}
		\includegraphics[width=0.95\columnwidth]{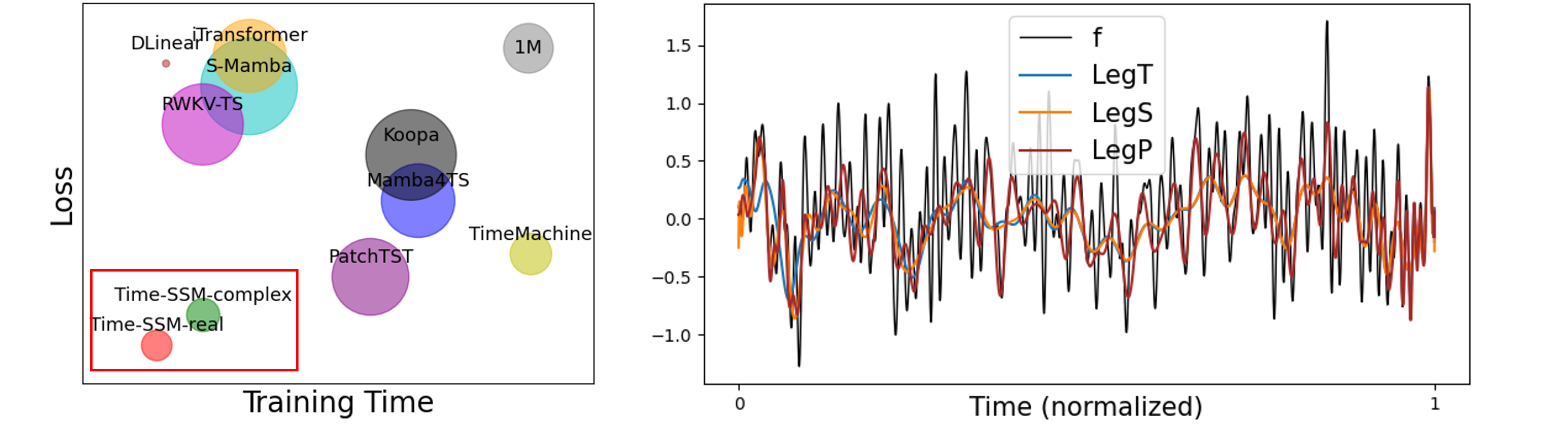}
\end{adjustbox}
\caption{Left: Complexity analysis. Right: Long-range function approximation with different SSM basis}
\label{complexity}
\end{figure}

\noindent \textbf{Efficiency.}~
As shown in Figure \ref{complexity} (left), we present a comparison of the complexity between Time-SSM, and other advanced TSF models. The radius of the circle represents the number of model parameters. It can be observed that overall, the Time-SSM exhibits superior efficiency. Although the DLinear model has fewer model parameters, its performance significantly lags behind the Time-SSM. In practical applications, the Time-SSM in the real plane is undoubtedly the optimal choice.

\noindent \textbf{Representation Ability.}~
We utilize a 64-dimensional (same as in Time-SSM) fixed Hippo matrix to reconstruct the input function, in order to explore the representational capacity of different SSM bases. The input function is randomly sampled from a continuous-time band-limited white noise process with a length of $10^6$, a sampling step size of $10^{-4}$, and a signal band limit of 1Hz. From Figure \ref{complexity} (right), we observe that: (a) Consistent with Corollary \ref{legp good}, LegP (with a piecewise scale of 3) exhibits stronger representational capacity, outperforming LegT and LegS in limited dimensions. (b) According to previous research \cite{zhou2022film, hu2024attractor, gu2020hippo}, a spectral space of 256 dimensions is typically required to effectively represent temporal dynamics in single-channel SSMs. However, in our experiments, we found that Time-SSM with a 64-dimensional spectral space performs better. We believe this is because the Time-SSM first obtains a high-dimensional vectorized representation of the time series through patch embedding. In the context of MIMO SSMs, the representational capacity of the Time-SSM is multiplicatively accumulated (N plus D), which is sufficient for dynamic modeling.
\section{Conclusion}
This paper aims to address the theoretical and physical limitations of SSM in the TSF context, providing comprehensive theoretical guidance for its application in time series data. To achieve this, we propose the theory of dynamic spectral operators based on the concept of IOSSM, and introduce a simple and efficient foundational model called Time-SSM. Additionally, we explore various variants of SSM kernels in both the real and complex planes. Through an extensive range of experiments, including ablation studies, autoregressive testing, complexity analysis, and white noise reconstruction, we rigorously validate and explain the effectiveness of our theory and proposed Time-SSM model. Ultimately, our goal is to contribute to the flourishing development of SSM not only within the TSF community but also in other machine-learning domains. Limitations can be found in Appendix \ref{limitation}.

\bibliography{reference}

\clearpage
\addtocontents{toc}{\protect\setcounter{tocdepth}{2}}

\appendix
\onecolumn

\section{Model Details}
\subsection{Computation for SSM-LegP}
\label{compute legp}
It is first proposed by Attraos \cite{hu2024attractor}. We can perform computations either explicitly or implicitly.

\begin{figure}[h]
	\centering
    {\includegraphics[width=.6\columnwidth]{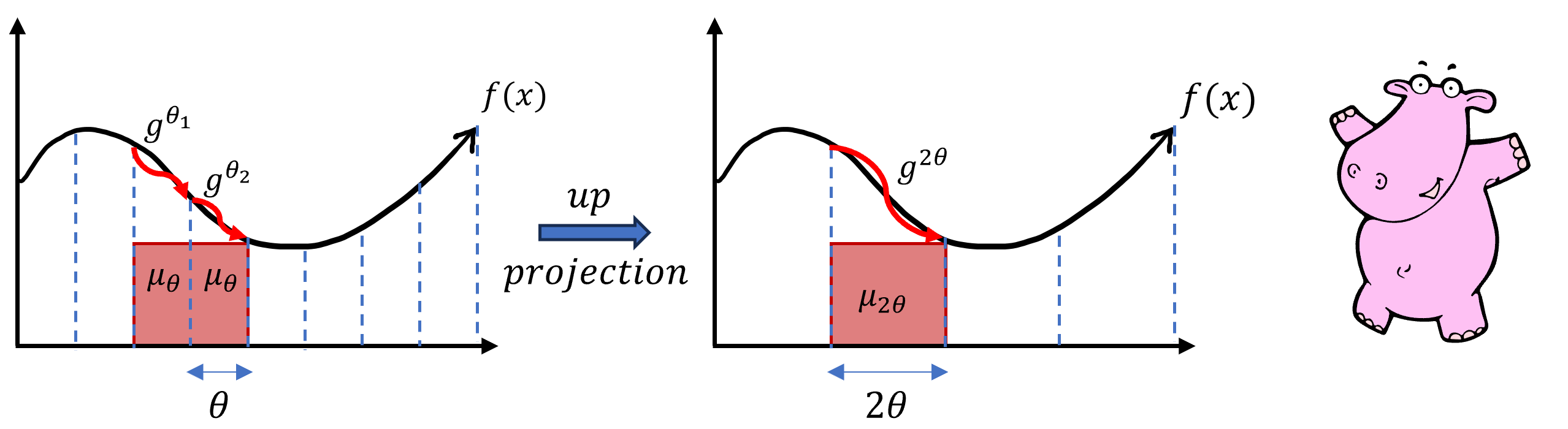}}
    \vspace{-2mm}
	\caption{\cite{hu2024attractor} Elaboration of Hippo-LegP}
    \label{hippo-legp}
    \vspace{-2mm}
\end{figure}
As shown in Figure \ref{hippo-legp} (top), In Hippo-LegT, we aim to approximate the input function in every measure window $\theta$. In Hippo-LegP, we want this approximation to be multi-scale. Specifically, when $p_n^{m\tau}(t, s) =  2^{m/2}L_n(2^{m}t-\tau)$, the piecewise window size will expand by power of 2 ($\theta \rightarrow 2\theta$).

When the window length is set to 2$\theta$, the region previously approximated by $g^{\theta_1}\in \mathcal{G}_\theta^k$ (left half) and $g^{\theta_2}\in \mathcal{G}_{\theta}^k$ (right half) will now be approximated by $g^{2\theta}\in \mathcal{G}_{2\theta}^k$. 
Since the piecewise polynomial function space can be defined as the following form:
\begin{equation}
\mathcal{G}_{(r)}^k= 
\begin{cases}
g \mid \operatorname{deg}(g)<k, &x \in\left(2^{-r} l, 2^{-r}(l+1)\right)\\ 
0, &\text{otherwise}
\end{cases},
\label{Allocation}
\end{equation}
with polynomial order $k \in \mathbb{N}$, piecewise scale $r \in \mathbb{Z}^{+} \cup\{0\}$, and piecewise internal index $l\in\{0,1,...,2^r-1\}$, it is evident that $dim(\mathcal{G}_{(r)}^k)=2^rk$, implying that $\mathcal{G}_\theta^k$ possesses a superior function capacity compared to $\mathcal{G}_{2\theta}^k$. All functions in $\mathcal{G}_{2\theta}^k$ are encompassed within the domain of $\mathcal{G}_\theta^k$. Moreover, since $\mathcal{G}_{\theta}^k$ and $\mathcal{G}_{2\theta}^k$ can be represented as space spanned by basis functions $\{\phi_i^{\theta}(x)\}$ and $\{\phi_i^{2\theta}(x)\}$, any function including the basis function within the $\mathcal{G}_{2\theta}^k$ space can be precisely expressed as a linear combination of basis functions from the $\mathcal{G}_\theta^k$ space with a proper tilted measure ${\mu_{2\theta}}$:
\begin{equation}
    \phi_i^{2\theta}(x) = \sum_{j=0}^{k-1} H_{i j}^{\theta_1}\phi_i^{\theta}(x)_{x\in{[\theta_1]}}+ \sum_{j=0}^{k-1} H_{i j}^{\theta_2}\phi_i^{\theta}(x)_{x\in{[\theta_2]}},
    \label{phi_proj}
\end{equation}
and back-projection can be achieved through the Moore-Penrose inverse $H^\dag$. By taking the inner product with $f$ on both sides in Eq. \ref{phi_proj}, we can project the state representation $x(t)$ between $\mathcal{G}_\theta^k$ space and the $\mathcal{G}_{2\theta}^k$ space by considering the odd and even positions along the $L$ dimension in $x$:
\begin{equation}
\centering
\begin{aligned}
    x_t^{2\theta}&=H^{\theta_1} x_t^{\theta_1} +H^{\theta_2}x_t^{\theta_2}. \quad x^{2\theta}\in\mathbb{R}^{B\times L/2 \times D \times N}\\
    x_t^{\theta_1}&=H^{\dag\theta_1} x_t^{2\theta}, \quad x_t^{\theta_2}=H^{\dag\theta_2}x_t^{2\theta}.
\end{aligned}
\end{equation}
$H$ and $H^\dag$ can be achieved by Gaussian Quadrature \cite{gupta2021multiwavelet}. Iteratively repeating this process enables us to model the temporal dynamic from a more macroscopic perspective. 

\begin{remark}
    Although $\bm{\dt}$ as a kind of attention mechanism leads to different measure windows, ie., $\theta_1\neq \theta_2$, it still maintains the linear projection property for up and down projection:$\mathcal{G}_{\theta_1},\mathcal{G}_{\theta_2}\leftrightarrow \mathcal{G}_{\theta_1+\theta_2}$. In our illustration, we have used a unified measure window for simplicity.
\end{remark}

We can firstly obtain the state representation $x\in\mathbbm{R/C}^{L}$ by Equation \ref{eq:ssm-a}, and iteratively perform up projection based on the sequence positions. In each scale, we use a unified $\bm{C}$ matrix to get output. Then, the back projection is employed to recover the sequence in length L. $H$, and $H^\dag$ are gradients opened to be adaptively optimized. It can be seen as an additional block after the SSM kernel to capture multi-scale information.




\subsection{Visual Explanation }
\label{visual}
\begin{figure}[h]
	\centering
    {\includegraphics[width=.49\columnwidth]{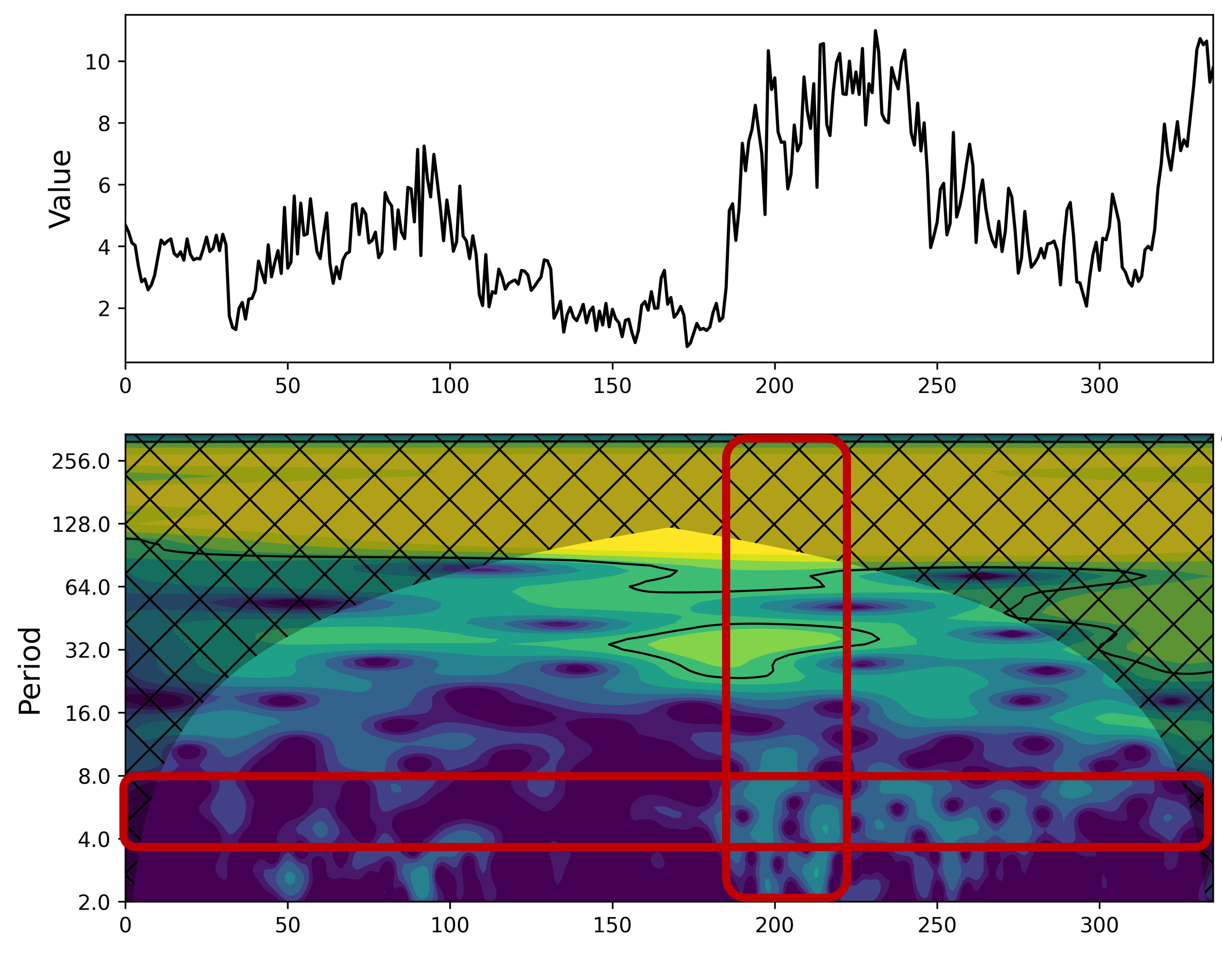}}
    {\includegraphics[width=.49\columnwidth]{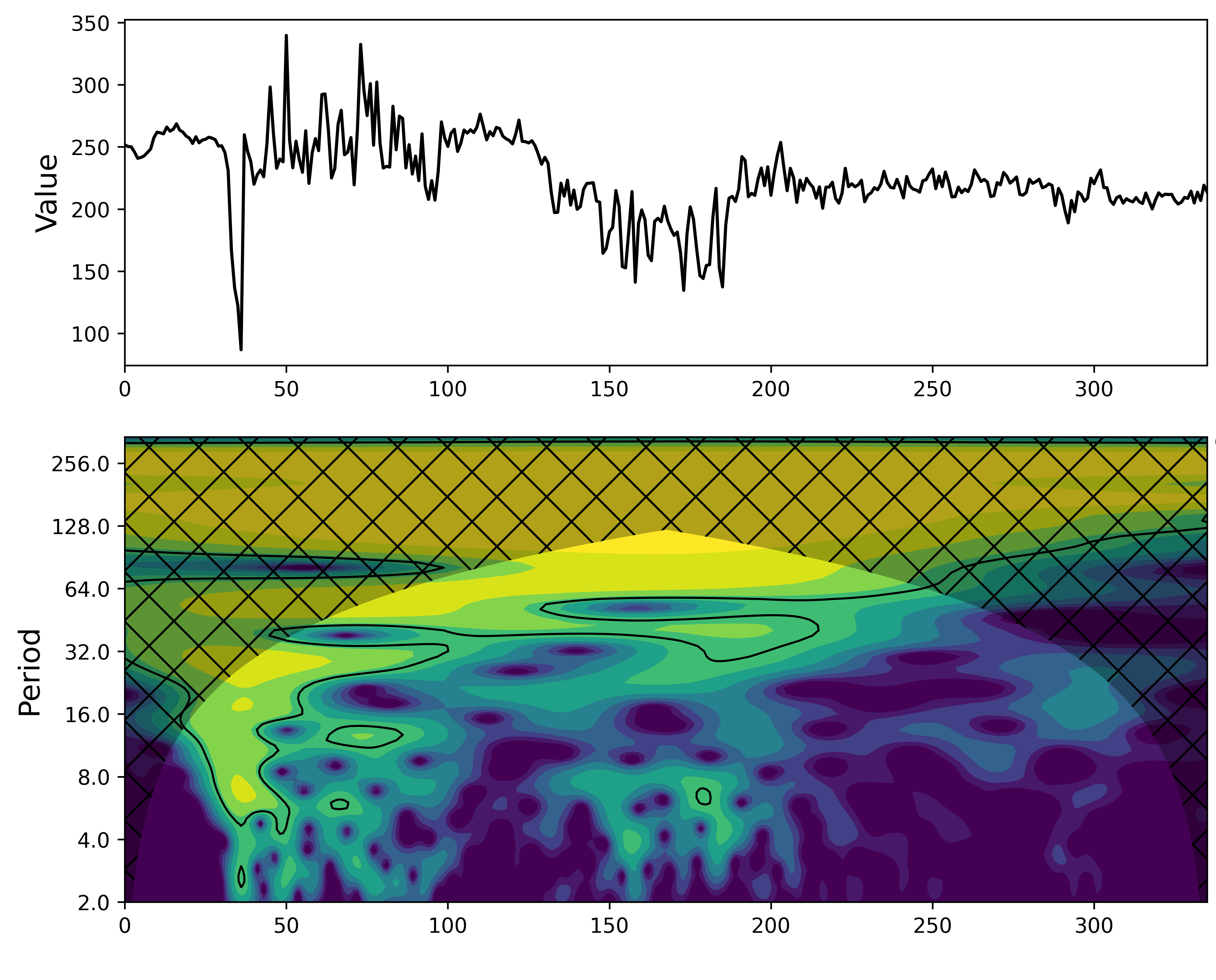}}
    \vspace{-2mm}
	\caption{\cite{hu2024twins} Examples of the wv and wd variables in the Weather dataset. The x-axis is the shared time steps. The graph above shows the variable values over time. The one below is the wavelet level plot, representing the strength of the signal's energy at different time-frequency scales.}
    \label{Multi wavelets analys}
    \vspace{-2mm}
\end{figure}
Due to the non-stationarity of real-world time series, their dynamics are often not uniformly distributed along the time axis. As shown in Figure \ref{Multi wavelets analys}, we visualize the energy level distribution of real-world time series using wavelet analysis, which effectively demonstrates this phenomenon. Therefore, both patch operations and the Hippo-LegP framework aim to model temporal dynamics in a piecewise approximation manner. Furthermore, we observe significant differences in the dynamics distribution among different variables or the existence of temporal delay phenomena. This is why modeling the relationships between variables can lead to significant performance improvements, particularly in small-scale datasets.
\section{Proofs}
\label{proof}
\begin{lemma} 
    For a differential equation of $x^{\prime}(t)=\bm{A} x(t)+\bm{B} u(t)$, its general solution is: 
$$
x(t)=e^{\bm{A}\left(t-t_0\right)} x\left(t_0\right)+\int_{t_0}^t e^{\bm{A}(t-s)} \bm{B} u(s) \mathrm{d}s.
$$
\end{lemma}
\begin{proof}
    \label{proof: general solution}
Let's start with the one-dimensional system represented by the scalar differential equation
\begin{equation}
    \dot{x}(t)=a x(t)+b u(t) \quad x\left(t_0\right)=x_0
\end{equation}
in which $a$ and $b$ are scalar constants, and $u(t)$ is a given scalar input signal. A traditional approach for deriving a solution formula for the scalar state $x(t)$ is to multiply both sides of the differential equation by the integrating factor $e^{-a\left(t-t_0\right)}$ to yield
$$
\begin{aligned}
\frac{d}{d t}\left(e^{-a\left(t-t_0\right)} x(t)\right) & =e^{-a\left(t-t_0\right)} \dot{x}(t)-e^{-a\left(t-t_0\right)} a x(t) \\
& =e^{-a\left(t-t_0\right)} b u(t)
\end{aligned}
$$
We next integrate from $t_0$ to $t$ and invoke the fundamental theorem of calculus to obtain
$$
\begin{aligned}
e^{-a\left(t-t_0\right)} x(t)-e^{-a\left(t-t_0\right)} x\left(t_0\right) & =\int_{t_0}^t \frac{d}{d t}\left(e^{-a\left(s-t_0\right)} x(s)\right) d s \\
& =\int_{t_0}^t e^{-a\left(s-t_0\right)} b u(s) d s .
\end{aligned}
$$

After multiplying through by $e^{a\left(t-t_0\right)}$ and some manipulation, we get
\begin{equation}
    x(t)=e^{a\left(t-t_0\right)} x_0+\int_{t_0}^t e^{a(t-s)} b u(s) d s
    \label{scaler solution}
\end{equation}
which expresses the state response $x(t)$ as a sum of terms, the first owing to the given initial state $x\left(t_0\right)=x_0$ and the second owing to the specified input signal $u(t)$. Notice that the first component characterizes the state response when the input signal is identically zero. We therefore refer to the first term as the zero-input response component. Similarly, the second component characterizes the state response for zero initial state, referred to as the zero-state response component.

\textbf{Now, let us derive a closed-form solution to the $n$-dimensional linear time-invariant state equation \eqref{eq:ssm}} given a specified initial state $x\left(t_0\right)=x_0$ and input vector $u(t)$. We begin with a related homogeneous matrix differential equation
\begin{equation}
\dot{X}(t)=A X(t) \quad X\left(t_0\right)=I
\label{2.6}
\end{equation}
where $I$ is the $n \times n$ identity matrix. We assume an infinite power series form for the solution
\begin{equation}
X(t)=\sum_{k=0}^{\infty} X_k\left(t-t_0\right)^k
\label{2.7}
\end{equation}

Each term in the sum involves an $n \times n$ matrix $X_k$ to be determined and depends only on the elapsed time $t-t_0$, reflecting the time-invariance of the state equation. The initial condition for Equation \eqref{2.6} yields $X\left(t_0\right)=$ $X_0=I$. Substituting Equation \eqref{2.7} into Equation \eqref{2.6}, formally differentiating term by term with respect to time, and shifting the summation index gives
$$
\begin{aligned}
\sum_{k=0}^{\infty}(k+1) X_{k+1}\left(t-t_0\right)^k & =A\left(\sum_{k=0}^{\infty} X_k\left(t-t_0\right)^k\right) \\
& =\sum_{k=0}^{\infty} A X_k\left(t-t_0\right)^k
\end{aligned}
$$

By equating like powers of $t-t_0$, we obtain the recursive relationship
$$
X_{k+1}=\frac{1}{k+1} A X_k \quad k \geq 0
$$
which, when initialized with $X_0=I$, leads to
$$
X_k=\frac{1}{k!} A^k \quad k \geq 0
$$

Substituting this result into the power series \eqref{2.7} yields
$$
X(t)=\sum_{k=0}^{\infty} \frac{1}{k!} A^k\left(t-t_0\right)^k
$$

We note here that the infinite power series \eqref{2.7} has the requisite convergence properties so that the infinite power series resulting from term by-term differentiation converges to $\dot{X}(t)$, and Equation \eqref{2.6} is satisfied.

$$
\begin{aligned}
\sum_{k=0}^{\infty}(k+1) X_{k+1}\left(t-t_0\right)^k & =A\left(\sum_{k=0}^{\infty} X_k\left(t-t_0\right)^k\right) \\
& =\sum_{k=0}^{\infty} A X_k\left(t-t_0\right)^k
\end{aligned}
$$

By equating like powers of $t-t_0$, we obtain the recursive relationship
$$
X_{k+1}=\frac{1}{k+1} A X_k \quad k \geq 0
$$
which, when initialized with $X_0=I$, leads to
$$
X_k=\frac{1}{k!} A^k \quad k \geq 0
$$

Substituting this result into the power series \eqref{2.7} yields
$$
X(t)=\sum_{k=0}^{\infty} \frac{1}{k!} A^k\left(t-t_0\right)^k
$$

Note that the infinite power series \eqref{2.7} has the requisite convergence properties so that the infinite power series resulting from term-by-term differentiation converges to $\dot{X}(t)$, and Equation \eqref{2.6} is satisfied.

Recall that the scalar exponential function is defined by the following infinite power series
$$
\begin{aligned}
e^{a t} & =1+a t+\frac{1}{2} a^2 t^2+\frac{1}{6} a^3 t^3+\cdots \\
& =\sum_{k=0}^{\infty} \frac{1}{k!} a^k t^k
\end{aligned}
$$

Motivated by this, we define the so-called matrix exponential via
\begin{equation}
\begin{aligned}
    e^{A t} & =I+A t+\frac{1}{2} A^2 t^2+\frac{1}{6} A^3 t^3+\cdots \\
& =\sum_{k=0}^{\infty} \frac{1}{k!} A^k t^k
\end{aligned}
\label{2.8}
\end{equation}
Now, we can finally generalize \eqref{scaler solution} to Lemma \ref{general solution}. 

It is important to point out that $e^{A t}$ is merely a notation used to represent the power series in Equation \eqref{2.8}. Beyond the scalar case, the matrix exponential never equals the matrix of scalar exponentials corresponding to the individual elements in the matrix $A$. That is,
$$
e^{A t} \neq\left[e^{a_{i j} t}\right]
$$
\end{proof}

\begin{lemma}
    Any nonlinear, continuous differentiable dynamic
    \begin{equation}
        \dot{x}(t) =f(x(t), u(t), t)
        \label{nonlinear dynamic}
    \end{equation}
    can be represented by its linear nominal SSM ($\tilde{x}, \tilde{u}$) plus a third-order infinitesimal quantity:
    \small
$$
\dot{x}(t)=\bm{A} x(t)+\bm{B}u(t) +\mathcal{O}(x,u); ~~~~~ \bm{A}=\left.\frac{\partial f}{\partial x}(x, u)\right|_{\tilde{x}, \tilde{u}}; \bm{B}=\left.\frac{\partial f}{\partial u}(x, {u})\right|_{\tilde{x}, \tilde{u}}.
$$
\end{lemma}
\begin{proof}
    \label{proof: nolinear}



Expanding the nonlinear maps in Equation \eqref{nonlinear dynamic} in a multivariate Taylor series about $[\tilde{x}(t), \tilde{u}(t), t]$ we obtain
$$
\begin{aligned}
\dot{x}(t)= & f[x(t), u(t), t] \\
= & f[\tilde{x}(t), \tilde{u}(t), t]+\frac{\partial f}{\partial x}[\tilde{x}(t), \tilde{u}(t), t][x(t)-\tilde{x}(t)] \\
& +\frac{\partial f}{\partial u}[\tilde{x}(t), \tilde{u}(t), t][u(t)-\tilde{u}(t)]+\text { higher-order terms } 
\end{aligned}
$$

On defining the Jacobian matrix as the coefficient matrix
$$
\begin{aligned}
& A(t)=\frac{\partial f}{\partial x}(\tilde{x}(t), \tilde{u}(t), t) \\
& B(t)=\frac{\partial f}{\partial u}(\tilde{x}(t), \tilde{u}(t), t)
\end{aligned}
$$
rearranging slightly we have
$$
\dot{x}(t)=\bm{A} x(t)+\bm{B}u(t) +\mathcal{O}(x,u).
$$
Deviations of the state, input, and output from their nominal trajectories are denoted by $\delta$ subscripts
$$
\begin{aligned}
& x_\delta(t)=x(t)-\tilde{x}(t) \\
& u_\delta(t)=u(t)-\tilde{u}(t)
\end{aligned}
$$
We also have
$$
\dot{x}_\delta(t)=\bm{A}(t) x_\delta(t)+\bm{B}(t) u_\delta(t)+\mathcal{O}(x,u)
$$
\end{proof}

\begin{lemma} \textbf{(Approximation Error Bound)}
    Suppose that the function $f:[0,1] \in \mathbb{R}$ is $k$ times continuously differentiable, the piecewise polynomial $g$ approximates $f$ with mean error bounded as follows:
$$
\left\|f-g\right\| \leq 2^{-r k} \frac{2}{4^k k !} \sup _{x \in[0,1]}\left|f^{(k)}(x)\right|.
$$
\end{lemma}

\begin{proof}
\label{proof for approximation error}
Similar to \cite{alpert1993class}, we divide the interval $[0,1]$ into subintervals on which $g$ is a polynomial; the restriction of $g$ to one such subinterval $I_{r, l}$ is the polynomial of degree less than $k$ that approximates $f$ with minimum mean error. Also, the optimal $g$ can be regarded as the orthonormal projection $Q_r^k f$ onto $\mathcal{G}_{(r)}^k$. We then use the maximum error estimate for the polynomial, which interpolates $f$ at Chebyshev nodes of order $k$ on $I_{r, l}$.
We define $I_{r, l}=\left[2^{-r} l, 2^{-r}(l+1)\right]$ for $l=0,1, \ldots, 2^r-1$, and obtain
$$
\begin{aligned}
\left\|Q_r^k f-f\right\|^2 & =\int_0^1\left[\left(Q_r^k f\right)(x)-f(x)\right]^2 d x \\
& =\sum_l \int_{I_{r, l}}\left[\left(Q_r^k f\right)(x)-f(x)\right]^2 d x \\
& \leq \sum_l \int_{I_{r, l}}\left[\left(C_{r, l}^k f\right)(x)-f(x)\right]^2 d x \\
& \leq \sum_l \int_{I_{r, l}}\left(\frac{2^{1-r k}}{4^k k !} \sup _{x \in I_{r, l}}\left|f^{(k)}(x)\right|\right)^2 d x \\
& \leq\left(\frac{2^{1-r k}}{4^k k !} \sup _{x \in[0,1]}\left|f^{(k)}(x)\right|\right)^2
\end{aligned}
$$
and by taking square roots, we have bound (7). Here $C_{r, l}^k f$ denotes the polynomial of degree $k$, which agrees with $f$ at the Chebyshev nodes of order $k$ on $I_{r, l}$, and we have used the well-known maximum error bound for Chebyshev interpolation.

The error of the approximation $Q_r^k f$ of $f$ therefore decays like $2^{-r k}$ and, since $S_r^k$ has a basis of $2^r k$ elements, we have convergence of order $k$. For the generalization to $m$ dimensions in the dynamic structure modeling, a similar argument shows that the rate of convergence is of order $k / m$.
\end{proof}

\begin{proposition} The process of gradient updates in the IOSSM($\bm{ABC}$) is a spectral space skewing:
$$span\{\phi_i\} \rightarrow span\{\psi_i\}, \quad \psi_i=\chi(t,s)\circ\phi_i.$$
\end{proposition}
\begin{proof}
    \label{proof: space skewing}
It can be regarded as a basis skewing from an orthogonal basis to an arbitrary basis by skewing in the basis itself or the measure, as shown in Hippo \cite{gu2020hippo}.

For any scaling function
$$
\chi(t, x)=\chi^{(t)}(x),
$$
the functions $p_n(x) \chi(x)$ are orthogonal with respect to the density $\omega / \chi^2$ at every time $t$. Thus, we can choose this alternative basis and measure to perform the projections.

To formalize this tilting with $\chi$, define $\nu^{(t)}$ to be the normalized measure with density proportional to $\omega^{(t)} /\left(\chi^{(t)}\right)^2$.
We will calculate the normalized measure and the orthonormal basis for it. Let
$$
\zeta(t)=\int \frac{\omega}{\chi^2}=\int \frac{\omega^{(t)}(x)}{\left(\chi^{(t)}(x)\right)^2} \mathrm{~d} x
$$
be the normalization constant, so that $\nu^{(t)}$ has density $\frac{\omega^{(t)}}{\zeta(t)\left(\chi^{(t)}\right)^2}$. If $\chi(t, x)=1$ (no tilting), this constant is $\zeta(t)=1$. In general, we assume that $\zeta$ is constant for all $t$; if not, it can be folded into $\chi$ directly.
Next, note that (dropping the dependence on $x$ inside the integral for shorthand)
$$
\begin{aligned}
\left\|\zeta(t)^{\frac{1}{2}} p_n^{(t)} \chi^{(t)}\right\|_{\nu^{(t)}}^2 & =\int\left(\zeta(t)^{\frac{1}{2}} p_n^{(t)} \chi^{(t)}\right)^2 \frac{\omega^{(t)}}{\zeta(t)\left(\chi^{(t)}\right)^2} \\
& =\int\left(p_n^{(t)}\right)^2 \omega^{(t)} \\
& =\left\|p_n^{(t)}\right\|_{\mu^{(t)}}^2=1 .
\end{aligned}
$$

Thus we define the orthogonal basis for $\nu^{(t)}$
$$
g_n^{(t)}=\lambda_n \zeta(t)^{\frac{1}{2}} p_n^{(t)} \chi^{(t)}, \quad n \in \mathbb{N} .
$$

We let each element of the basis be scaled by a $\lambda_n$ scalar, for reasons discussed soon, since arbitrary scaling does not change orthogonality:
$$
\left\langle g_n^{(t)}, g_m^{(t)}\right\rangle_{\nu^{(t)}}=\lambda_n^2 \delta_{n, m}
$$

Note that when $\lambda_n= \pm 1$, the basis $\left\{g_n^{(t)}\right\}$ is an orthonormal basis with respect to the measure $\nu^{(t)}$, at every time $t$. Notationally, let $g_n(t, x):=g_n^{(t)}(x)$ as usual.

\end{proof}

\begin{lemma} For the n-dimensional SSM($\bm{A}$,$\bm{B}$,$\bm{C}$), any transformation defined by a nonsingular matrix $\bm{T}$ will generate a transformed SSM($\hat{\bm{A}}$,$\hat{\bm{B}}$,$\hat{\bm{C}}$):
$$
\hat{\bm{A}}=\bm{T}^{-1} \bm{A} \bm{T} \quad \hat{\bm{B}}=\bm{T}^{-1} \bm{B} \quad \hat{\bm{B}}=\bm{C} \bm{T},
$$
where the dynamic matrix $\hat{\bm{A}}$ has the same characteristic polynomial and eigenvalues (same dynamics) as $\bm{A}$, but its eigenvectors are different. 
\end{lemma}
\begin{proof}
    \label{proof: translation ssm}
For the transformation
$$
x(t)=\bm{T} z(t) \quad z(t)=\bm{T}^{-1} x(t)
$$
We have
$$
\begin{aligned}
\dot{z}(t) & =\bm{T}^{-1} \dot{x}(t) \\
& =\bm{T}^{-1}[\bm{A} x(t)+\bm{B} u(t)] \\
& =\bm{T}^{-1} \bm{A} \bm{T} z(t)+\bm{T}^{-1} \bm{B} u(t)
\end{aligned}
$$

\end{proof}

\begin{definition} \textbf{(K-Lipschitz Continuity)}
Suppose that the function $f:[0,1] \in \mathbb{R}$ is $k$ times continuously differentiable, the piecewise polynomial $g\in\mathcal{G}_{r}^k$ approximates $f$ with mean error bounded as follows:
$$
\left\|f\left(x_1\right)-f\left(x_2\right)\right\| \leq K\left\|x_1-x_2\right\|, \forall x_1, x_2 \in {dom} f
$$
\end{definition}

\begin{lemma} The TOSSM($\bm{ABC}$) is K-Lipschitz with $\operatorname*{eig}(\bm{A})\leq K^2$.
\end{lemma}
\begin{proof}
    \label{proof: lipschi}
For a complex function, it has
$$
\|f \circ g\|_{\text {Lip }} \leq\|f\|_{\text {Lip }} \cdot\|g\|_{\text {Lip }}
$$
The essence of SSM lies in computing the Krylov kernel of matrix $\bm{A}$, so if we want a multi-layer SSM to be 1-Lipschitz, the $\bm{A}$ needs to be 1-Lipschitz.

For any $\bm{A}$, we have
$$\begin{aligned}
& \bm{A} \in \mathrm{K}-\text { Lipschitz, } \forall \bm{A}: R^n \rightarrow R^m / \forall \bm{A} \in R^{m \times n} \\
\Longleftrightarrow & \|\bm{A} \vec{x}\| \leq K\|\vec{x}\|, \forall \vec{x} \in R^n \\
\Longleftrightarrow & \langle \bm{A} \vec{x}, \bm{A} \vec{x}\rangle \leq K^2\langle\vec{x}, \vec{x}\rangle \\
\Longleftrightarrow & \vec{x}^T\left(\bm{A}^T \bm{A}-K^2 I\right) \vec{x} \leq 0 \\
\Longleftrightarrow & \left\langle\left(\bm{A}^T \bm{A}-K^2 I\right) \vec{x}, \vec{x}\right\rangle \leq 0
\end{aligned}$$

Since $\bm{A}^T\bm{A}$ is a real symmetric matrix, the eigenvectors corresponding to different eigenvalues are orthogonal to each other. Moreover, since $\bm{A}^T\bm{A}$ is positive semidefinite, all its eigenvalues are non-negative. Let's assume that the eigenvectors of $\bm{A}^T\bm{A} \in \mathbb{R}^{n \times n}$ are denoted as $\overrightarrow{v_i}$ for $i=1,2,\ldots,n$, with corresponding eigenvalues $\lambda_i$ for $i=1,2,\ldots,n$. Since $\overrightarrow{v_i}$ are pairwise orthogonal (not rigorously, as there may not be exactly $n$ of them), we can assume that $\overrightarrow{v_i}$ has already been normalized to unit length, thus forming an orthonormal basis for the $\mathbb{R}^n$ space. Therefore, for any $\vec{x} \in \mathbb{R}^n$, we can express it as $\vec{x}=\sum_{i=1}^n x_i \overrightarrow{v_i}$. we have:
$$
\begin{array}{ll} 
& \left\langle\left(A^T A-K^2 I\right) \vec{x}, \vec{x}\right\rangle \leq 0 \\
\Longleftrightarrow \quad & \left\langle\left(A^T A-K^2 I\right) \sum_{i=1}^n x_i \overrightarrow{v_i}, \sum_{i=1}^n x_i \overrightarrow{v_i}\right\rangle \leq 0 \\
\Longleftrightarrow \quad & \sum_{i=1}^n \sum_{j=1}^n x_i x_j\left\langle\left(A^T A-K^2 I\right) \overrightarrow{v_i}, \overrightarrow{v_j}\right\rangle \leq 0 \\
\Longleftrightarrow \quad & \sum_{i=1}^n \sum_{j=1}^n x_i x_j\left[\left(A^T A-K^2 I\right) \overrightarrow{v_i}\right]^T \overrightarrow{v_j} \leq 0
\end{array}
$$

because
$$
\begin{aligned}
\left[\left(A^T A-K^2 I\right) \overrightarrow{v_i}\right]^T \overrightarrow{v_j}={\overrightarrow{v_i}}^T\left(A^T A-K^2 I\right) \overrightarrow{v_j} = \overrightarrow{v_i}^T A^T A\overrightarrow{v_j} - K^2\overrightarrow{v_i}^T\overrightarrow{v_j} \\
= \begin{cases}0, & \text { if } i=1 j \\ \left(\lambda_i-K^2\right)\left\langle\overrightarrow{v_i}, \overrightarrow{v_i}\right\rangle=\lambda_i-K^2, & \text { if } i=j\end{cases}
\end{aligned}
$$

so we have
$$
\sum_{i=1}^n \sum_{j=1}^n x_i x_j\left[\left(A^T A-K^2 I\right) \overrightarrow{v_i}\right]^T \overrightarrow{v_j} \leq 0 \Longleftrightarrow \sum_{i=1}^n x_i^2\left(\lambda_i-K^2\right) \leq 0
$$

specifically
$$
A \in \text { K-Lipschitz }, \forall A: R^n \rightarrow R^m / \forall A \in R^{m \times n} \Longleftrightarrow \sum_{i=1}^n x_i^2\left(K^2-\lambda_i\right) \geq 0
$$

So the TOSSM($\bm{ABC}$) is K-Lipschitz with $\operatorname*{eig}(\bm{A})\leq K^2$.
\end{proof}
\section{Additional Experiments} \label{add exp}
\label{expdetails}
\subsection{Datasets}
\label{appendix:dataset}
We conduct experiments on 8 real-world datasets to evaluate the performance of our model, and the details of these datasets are provided in Table \ref{tab:dataset}. Dimension denotes the variate number of each dataset. Dataset Size denotes the total number of time points in the train, Validation, and Test) split, respectively. Forecasting Length denotes the future time points to be predicted, and four prediction settings are included in each dataset. Frequency denotes the sampling interval of time points.

Specifically, \textbf{ETT} \cite{zhou2021informer} contains 7 factors of electricity transformer from July 2016 to July 2018. We use four subsets where ETTh1 and ETTh2 is recorded every hour, ETTm1 and ETTm2 is recorded every 15 minutes. \textbf{Exchange}~\cite{wu2021autoformer} collects the panel data of daily exchange rates from 8 countries from 1990 to 2016. \textbf{Crypots} comprises historical transaction data for various cryptocurrencies, including Bitcoin and Ethereum. We select samples with $Asset\_ID$ set to 0, and remove the column $Count$. We download the data from \url{https://www.kaggle.com/competitions/g-research-crypto-forecasting/data?select=supplemental_train.csv}. \textbf{Weather} \cite{wu2021autoformer} includes 21 meteorological factors collected every 10 minutes from the Weather Station of the Max Planck Bio-geochemistry Institute in 2020. \textbf{Air Convection\footnote{https://www.psl.noaa.gov/}} dataset is obtained by scraping the data from NOAA, and it includes the data for the entire year of 2023. The dataset consists of 10 variables, including air humidity, pressure, convection characteristics, and others. The data was sampled at intervals of 15 minutes and averaged over the course of the entire year.

\begin{table}[!ht]
\centering
\caption{Detailed dataset descriptions.}
\small
\begin{tabular}{@{} c c c c c @{}}
\toprule
\textbf{Dataset} & {\textbf{Dimension}} & \textbf{Forecasting Length} & \textbf{Dataset Size} & \textbf{Information (Frequency)} \\
\midrule
ETTm1 & 7 & \{96, 192, 336, 720\} & (34369, 11425, 11425) & Electricity (15 min) \\
ETTh1 & 7 & \{96, 192, 336, 720\} & (8445, 2785, 2785) & Electricity (Hourly) \\
ETTm2 & 7 & \{96, 192, 336, 720\} & (34369, 11425, 11425) & Electricity (15 min) \\
ETTh2 & 7 & \{96, 192, 336, 720\} & (8545, 2881, 2881) & Electricity (Hourly) \\
Exchange      & 8 & \{96, 192, 336, 720\} & (5120, 665, 1422) & Exchange rate (Daily) \\
Crypto & 7 & \{96, 192, 336, 720\} & (125699, 17891, 35873) & Finance (1 min) \\
Weather       & 21 & \{96, 192, 336, 720\} & (36696, 5175, 10440) & Weather (10 min) \\
Air Convection  & 10 & \{96, 192, 336, 720\} & (14647, 4882, 4882) & Weather (15 min) \\
\bottomrule
\end{tabular}

\label{tab:dataset}
\end{table}

\subsection{Baselines}
\textbf{Mamba4TS:} Mamba4TS is a novel SSM architecture tailored for TSF tasks, featuring a parallel scan (\href{https://github.com/alxndrTL/mamba.py/tree/main}{https://github.com/alxndrTL/mamba.py/tree/main}). Additionally, this model adopts a patching operation with both patch length and stride set to 16. We use the recommended configuration as our experimental settings with a batch size of 32, and the learning rate is 0.0001.

\textbf{S-Mamba~\cite{wang2024mamba}:} S-Mamba utilizes a linear tokenization of variates and a bidirectional Mamba layer to efficiently capture inter-variate correlations and temporal dependencies. This approach underscores its potential as a scalable alternative to Transformer technologies in TSF. We download the source code from: \href{https://github.com/wzhwzhwzh0921/S-D-Mamba}{https://github.com/wzhwzhwzh0921/S-D-Mamba} and adopt the recommended setting as its experimental configuration.

\textbf{RWKV-TS~\cite{hou2024rwkv}:} RWKV-TS is an innovative RNN-based architecture for TSF that offers linear time and memory efficiency. We download the source code from: \href{https://github.com/howard-hou/RWKV-TS}{https://github.com/howard-hou/RWKV-TS}. We follow the recommended settings as experimental configuration.



\textbf{TimeMachine~\cite{atik2024timemachine}:} TimeMachine utilizes a unique quadruple-Mamba architecture to handle both channel-mixing and channel-independence effectively, allowing for precise content selection based on contextual cues at multiple scales. We download the source code from: \href{https://github.com/Atik-Ahamed/TimeMachine}{https://github.com/Atik-Ahamed/TimeMachine}. We use the recommended configuration as the experimental settings where the batch size is 64, the learning rate is 0.001 to train the model in 10 epochs for fair comparison.

\textbf{Koopa~\cite{liu2024koopa}:} Koopa is a novel forecasting model that tackles non-stationary time series using Koopman theory to differentiate time-variant dynamics. It features a Fourier Filter and Koopman Predictors within a stackable block architecture, optimizing hierarchical dynamics learning. We download the source code from: \href{https://github.com/thuml/Koopa}{https://github.com/thuml/Koopa}. We set the lookback window to fixed values of \{96, 192, 336, 720\} instead of twice the output length as in the original experimental settings.

\textbf{iTransformer~\cite{liu2023itransformer}:} iTransformer modifies traditional Transformer models for time series forecasting by inverting dimensions and applying attention and feed-forward networks across variate tokens. We download the source code from: \href{https://github.com/thuml/iTransformer}{https://github.com/thuml/iTransformer}. We follow the recommended settings as experimental configuration.

\textbf{PatchTST~\cite{nie2022time}:} PatchTST introduces a novel design for Transformer-based models tailored to time series forecasting. It incorporates two essential components: patching and channel-independent structure. We obtain the source code from: \href{https://github.com/PatchTST}{https://github.com/PatchTST}. This code serves as our baseline for long-term forecasting, and we follow the recommended settings for our experiments.

\textbf{LTSF-Linear~\cite{zeng2023transformers}:} In LTSF-Linear family, DLinear decomposes raw data into trend and seasonal components and NLinear is just a single linear models to capture temporal relationships between input and output sequences. We obtain the source code from: \href{https://github.com/cure-lab/LTSF-Linear}{https://github.com/cure-lab/LTSF-Linear}, using it as our long-term forecasting baseline and adhering to recommended settings for experimental configuration.

\textbf{TimesNet~\cite{wu2023timesnet}:} TimesNet is a task-generic framework for time series analysis that processes complex temporal variations by transforming one-dimensional time series into two-dimensional tensors based on the observation of multi-periodicity in time series. We obtain the source code from: \href{https://github.com/thuml/Time-Series-Library}{https://github.com/thuml/Time-Series-Library} and follow the recommended settings.

\textbf{GPT4TS~\cite{zhou2023one}:} This study explores the application of pre-trained language models to time series analysis tasks, demonstrating that the Frozen Pretrained Transformer (FPT), without modifications to its core architecture, achieves state-of-the-art results across various tasks. We download the source code from: \href{https://github.com/DAMO-DI-ML/NeurIPS2023-One-Fits-All}{https://github.com/DAMO-DI-ML/NeurIPS2023-One-Fits-All}. We follow the recommended settings as experimental configuration.

\textbf{Crossformer~\cite{zhang2023crossformer}:} Crossformer explicitly exploits dependencies across channel dimensions and time steps. It introduces a two-stage attention mechanism for learning dependencies and achieves linear complexity in its operations. We train the model integrated in: \href{https://github.com/thuml/Time-Series-Library}{https://github.com/thuml/Time-Series-Library}.

\textbf{Scaleformer~\cite{shabani2023scaleformer}:} Scaleformer introduces a multi-scale framework that improves the performance of transformer-based time series forecasting models, achieving higher accuracy with minimal additional computational cost. We download the source code from: \href{https://github.com/BorealisAI/scaleformer}{https://github.com/BorealisAI/scaleformer} and use the recommended configuration as the experimental settings. 

\textbf{TiDE~\cite{das2023tide}:} TiDE is an MLP-based encoder-decoder model that combines the lightweight and fast training attributes of linear models while also handling covariates and nonlinear dependencies. We train the model integrated in: \href{https://github.com/thuml/Time-Series-Library}{https://github.com/thuml/Time-Series-Library}.

\textbf{Stationary~\cite{liu2022non}:} Non-stationary Transformer focuses on exploring the stabilization issues in time series forecasting tasks. We train the model integrated in: \href{https://github.com/thuml/Time-Series-Library}{https://github.com/thuml/Time-Series-Library} and use the recommended configuration as the experimental settings.

\textbf{FEDformer~\cite{zhou2022fedformer}:} FEDformer introduces an attention mechanism with a low-rank approximation in the frequency domain, along with a mixed decomposition to control distribution shifts. We adheres to the recommended settings for the experimental configuration and download the code from \href{https://github.com/thuml/Time-Series-Library}{https://github.com/MAZiqing/FEDformer}.

\subsection{Experiments Setting}
All experiments are conducted on the NVIDIA RTX3090-24G and A6000-48G GPUs. The Adam optimizer is chosen. A grid search is performed to determine the optimal hyperparameters, including the learning rate from \{0.0001, 0.0005, 0.001\} and patch length from \{8, 16, 24\}. The stride is equal to the patch length. The primary mode in the variable kernel is 64, the hidden dim is 256, the state space dim is 64, the piecewise scale is based on $log2(L)$ (L is the patches number), the model layers are 2, and the input length is 96. All experimental results are averaged over two runs.

\subsection{Full Results}\label{full result}
As shown in Table \ref{full 96}, we provide the full experiment results.
\begin{table*}[h!]
\centering
\caption{Multivariate long-term series forecasting results in mainstream datasets with an input length of 96 and prediction horizons of \{96, 192, 336, 720\}. \textbf{{\color{red} Red}}: best, \textbf{{\color{blue} Blue}}: second best. The bottom part introduces the variable Kernel for multivariate variables.}\label{full 96}
\setlength{\tabcolsep}{10pt}
\begin{adjustbox}{width=1\columnwidth, center}
\resizebox{1\columnwidth}{!}{
\begin{tabular}{p{0.3cm}|c|cc|cc|cc||cc|cc|cc|cc|cc|cc}
\toprule[1.5pt]
\multicolumn{2}{c}{\multirow{2}{*}{\large{Model}}}&\multicolumn{2}{c}{\cellcolor{mycolor4}{Time-SSM}}&\multicolumn{2}{c}{\cellcolor{mycolor4}Mamba4TS}&\multicolumn{2}{c}{\cellcolor{mycolor4}S-Mamba}&\multicolumn{2}{c}{\cellcolor{mycolor4}RWKV-TS} &\multicolumn{2}{c}{\cellcolor{mycolor4}\color{mypurple}\textbf{TimeMachine}}&\multicolumn{2}{c}{\cellcolor{mycolor4}Koopa}&\multicolumn{2}{c}{\cellcolor{mycolor4}InvTrm}&\multicolumn{2}{c}{\cellcolor{mycolor4}PatchTST}&\multicolumn{2}{c}{\cellcolor{mycolor4}DLinear}\\
\multicolumn{2}{c}{} & \multicolumn{2}{c}{\cellcolor{mycolor4}{(Ours)}} & \multicolumn{2}{c}{\cellcolor{mycolor4}(Temporal Emb.)}&\multicolumn{2}{c}{\cellcolor{mycolor4}(Channel Emb.\cite{wang2024mamba})} & \multicolumn{2}{c}{\cellcolor{mycolor4}\cite{hou2024rwkv}} & \multicolumn{2}{c}{\cellcolor{mycolor4}\cite{atik2024timemachine}} & \multicolumn{2}{c}{\cellcolor{mycolor4}\cite{liu2024koopa}} & \multicolumn{2}{c}{\cellcolor{mycolor4}\cite{liu2023itransformer}} & \multicolumn{2}{c}{\cellcolor{mycolor4}\cite{nie2022time}} & \multicolumn{2}{c}{\cellcolor{mycolor4}\cite{zeng2023transformers}} \\
\cmidrule(l){1-2}\cmidrule(l){3-4}\cmidrule(l){5-6}\cmidrule(l){7-8}\cmidrule(l){9-10}\cmidrule(l){11-12}\cmidrule(l){13-14}\cmidrule(l){15-16}\cmidrule(l){17-18}\cmidrule(l){19-20}
\multicolumn{2}{c}{Metric}&MSE & MAE & MSE &MAE & MSE &MAE & MSE & MAE & MSE & MAE & MSE & MAE & MSE & MAE & MSE & MAE & MSE & MAE \\
\midrule
\midrule
\multirow{5}{*}{\begin{sideways}ETTh1\end{sideways}} 
& 96  & \textbf{\color{blue}0.377} & \textbf{\color{red}0.394} & 0.386 & 0.400 & 0.388 & 0.407 & 0.383 & 0.401 & 0.373 & 0.393 & 0.385 & 0.408 & 0.393 & 0.409 & \textbf{\color{red}0.375} & \textbf{\color{blue}0.396} & 0.396 & 0.410 \\
& 192 & \textbf{\color{red}0.423} & \textbf{\color{red}0.424} & \textbf{\color{blue}0.426} & 0.430 & 0.443 & 0.439 & 0.441 & 0.431 & 0.430 & 0.421 & 0.441 & 0.431 & 0.448 & 0.442 & 0.429 & \textbf{\color{blue}0.426} & 0.449 & 0.444 \\
& 336 & \textbf{\color{blue}0.466} & \textbf{\color{red}0.437} & 0.484 & 0.451 & 0.492 & 0.467 & 0.493 & 0.465 & 0.480 & 0.443 & 0.474 & 0.454 & 0.491 & 0.465 & \textbf{\color{red}0.461} & \textbf{\color{blue}0.448} & 0.487 & 0.465 \\
& 720 & \textbf{\color{red}0.452} & \textbf{\color{red}0.448} & 0.481 & 0.472 & 0.511 & 0.499 & 0.501 & 0.487 & 0.463 & 0.456 & 0.501 & 0.480 & 0.518 & 0.501 & \textbf{\color{blue}0.469} & \textbf{\color{blue}0.471} & 0.515 & 0.512 \\
& AVG & \textbf{\color{red}0.425} & \textbf{\color{red}0.426} & 0.444 & 0.438 & 0.459 & 0.453 & 0.454 & 0.446 & 0.437 & 0.428 & 0.450 & 0.443 & 0.463 & 0.454 & \textbf{\color{blue}0.434} & \textbf{\color{blue}0.435} & 0.462 & 0.458 \\
\midrule
\multirow{5}{*}{\begin{sideways}ETTh2\end{sideways}} 
& 96  & \textbf{\color{red}0.290} & \textbf{\color{red}0.341} & 0.297 & 0.347 & 0.296 &0.347 & \textbf{\color{red}0.290} & \textbf{\color{blue}0.342} & 0.284 & 0.336 & 0.317 & 0.359 & 0.302 & 0.351 & 0.295 & 0.344 & 0.353 & 0.405 \\
& 192 & \textbf{\color{red}}{0.368} & \textbf{\color{red}0.387} & 0.392 & 0.409 & 0.377 & 0.398 & \textbf{\color{blue}0.372} & \textbf{\color{blue}0.393} & 0.363 & 0.389 & 0.375 & 0.399 & 0.379 & 0.399 &0.375 & 0.399 & 0.482 & 0.479 \\
& 336 & \textbf{\color{red}0.416} & \textbf{\color{blue}0.430} & 0.424 & 0.436 & 0.425 & 0.435& \textbf{\color{blue}0.417} & 0.431 & 0.399 & 0.419 & 0.436 & 0.446 & 0.423 & 0.432 & 0.420 & \textbf{\color{red}0.429} & 0.588 & 0.539 \\
& 720 & \textbf{\color{blue}0.424} & \textbf{\color{red}0.439} & 0.431 & 0.448 & 0.427 & 0.446 & \textbf{\color{red}0.421} & \textbf{\color{blue}0.442} & 0.427 & 0.443 & 0.460 & 0.463 & 0.429 & 0.447 &0.431 & 0.451 & 0.833 & 0.658 \\
& AVG & \textbf{\color{red}0.374} & \textbf{\color{red}0.399} & 0.386 & 0.410 & 0.381 & 0.407 & \textbf{\color{blue}0.375} & \textbf{\color{blue}0.402} & 0.368 & 0.397 & 0.397 & 0.417 & 0.383 & 0.407 & 0.380 & 0.406 & 0.564 & 0.520 \\
\midrule
\multirow{5}{*}{\begin{sideways}ETTm1\end{sideways}} 
& 96  & 0.329 & \textbf{\color{blue}0.365} & 0.331 & 0.368 & 0.332 & 0.368 & 0.328 & 0.366 & 0.313 & 0.352 & \textbf{\color{red}0.322} & \textbf{\color{red}0.360} & 0.343 & 0.377 & \textbf{\color{blue}0.326} & \textbf{\color{blue}0.365} & 0.345 & 0.372 \\
& 192 & \textbf{\color{blue}0.370} & \textbf{\color{red}0.379} & 0.376 & 0.391 & 0.378 & 0.393 & 0.372 & 0.389 & 0.362 & 0.382 & 0.378 & 0.393 & 0.379 & 0.394 & \textbf{\color{red}0.365} & \textbf{\color{blue}0.383} & 0.382 & 0.391 \\
& 336 & \textbf{\color{red}0.396} & \textbf{\color{red}0.402} & 0.406 & 0.413 & 0.409 &0.414 & 0.401 & 0.409 & 0.391 & 0.400 & 0.405 & 0.413 & 0.418 & 0.418 & \textbf{\color{red}0.396} & \textbf{\color{blue}0.405} & 0.413 & 0.413 \\
& 720 & \textbf{\color{red}0.449} & \textbf{\color{blue}0.440} & 0.469 & 0.452  & 0.476 & 0.453 & 0.462& 0.446 & 0.446 & 0.435 & 0.473 & 0.447 & 0.488 & 0.458 & \textbf{\color{blue}0.458} & \textbf{\color{red}0.439} & 0.472 & 0.450 \\
& AVG & \textbf{\color{red}0.386} & \textbf{\color{red}0.396} & 0.396 & 0.406 & 0.399 & 0.407 &\textbf{\color{blue}0.391} & 0.403 & 0.378 & 0.392 & 0.395 & 0.403 & 0.407 & 0.412 & 0.403 & \textbf{\color{blue}0.398} & 0.403 & 0.406 \\
\midrule
\multirow{5}{*}{\begin{sideways}ETTm2\end{sideways}} 
& 96  & \textbf{\color{red}0.176} & \textbf{\color{red}0.260} & 0.186 & 0.268 & 0.182 &0.267 & 0.181 &0.264  & 0.175 & 0.255 & 0.180 & \textbf{\color{blue}0.261} & 0.184 & 0.269 & \textbf{\color{blue}0.177} & \textbf{\color{red}0.260} & 0.192 & 0.291 \\
& 192 & 0.246 & \textbf{\color{blue}0.305} & 0.261 & 0.320 &0.248 &0.309 & \textbf{\color{blue}0.245} & 0.307 & 0.236 & 0.297 & \textbf{\color{red}0.244} & \textbf{\color{red}0.304} & 0.253 & 0.313 & 0.246 & 0.308 & 0.284 & 0.360 \\
& 336 & 0.305 &0.344 & 0.331 & 0.366 & 0.312 &0.350 &0.306 & 0.344 & 0.302 & 0.340 & \textbf{\color{red}0.300} & \textbf{\color{red}0.340} & 0.313 & 0.351 & \textbf{\color{blue}0.302} & \textbf{\color{blue}0.343} & 0.371 & 0.420 \\
& 720 & \textbf{\color{blue}0.406} & 0.405 & 0.418 & 0.416 & 0.412 & 0.406 & \textbf{\color{blue}0.406} & 0.406 & 0.394 & 0.394 & \textbf{\color{red}0.398} & \textbf{\color{red}0.400} & 0.413 & 0.406 & 0.407 & 0.405 & 0.532 & 0.511 \\
& AVG & \textbf{\color{blue}0.283} & \textbf{\color{blue}0.328} & 0.299 & 0.343 & 0.289 & 0.333 & 0.285 & 0.330 & 0.277 & 0.323 & \textbf{\color{red}0.281} & \textbf{\color{red}0.326} & 0.291 & 0.335 & \textbf{\color{blue}0.283} & 0.329 & 0.345 & 0.396 \\
\midrule
\multirow{5}{*}{\begin{sideways}Exchange\end{sideways}} 
& 96  & \textbf{\color{red}0.083} & \textbf{\color{red}0.202} & \textbf{\color{blue}0.086} & \textbf{\color{blue}0.205} & 0.087 & 0.209 &0.129 & 0.256 & 0.087 & 0.203 & 0.093 & 0.215 & 0.097 & 0.222 & 0.093 & 0.212 & 0.094 & 0.227 \\
& 192 & \textbf{\color{red}0.170} & \textbf{\color{red}0.295} & \textbf{\color{blue}0.173} & \textbf{\color{blue}0.297} & 0.180 & 0.303 &0.231 & 0.346 & 0.180 & 0.300 & 0.189 & 0.313 & 0.184 & 0.309 & 0.201 & 0.319 & 0.185 & 0.325 \\
& 336 & 0.334 &0.418 & 0.340 & 0.423 & \textbf{\color{blue}0.330} & \textbf{\color{blue}0.417} &0.380 & 0.448 & 0.346 & 0.423 & 0.371 & 0.443 & \textbf{\color{red}0.327} & \textbf{\color{red}0.416} & 0.338 & 0.422 & \textbf{\color{blue}0.330} & 0.437 \\
& 720 & \textbf{\color{blue}0.824} & \textbf{\color{blue}0.677} & 0.855 & 0.696 & 0.860 & 0.700& 0.883 & 0.704 & 0.943 & 0.720 & 0.908 & 0.726 & 0.885 & 0.715 & 0.900 & 0.711 & \textbf{\color{red}0.774} & \textbf{\color{red}0.673} \\
& AVG & \textbf{\color{blue}0.352} & \textbf{\color{red}0.398} & 0.364 & \textbf{\color{blue}0.405} & 0.364 & 0.407 & 0.406 & 0.439 & 0.389 & 0.412 &  0.390 & 0.424 & 0.366 & 0.416 & 0.383 & 0.416 & \textbf{\color{red}0.346} & 0.416 \\
\midrule
\multirow{5}{*}{\begin{sideways}Crypto\end{sideways}} 
& 96  & \textbf{\color{blue}0.177} & 0.142 & 0.179 & 0.143 & 0.187 & 0.147 & \textbf{\color{red}0.176} & \textbf{\color{red}0.139} & 0.182 & 0.138 & 0.181 & 0.143 & 0.183 & 0.144 & \textbf{\color{blue}0.177} & \textbf{\color{blue}0.141} & 0.183 & 0.155 \\
& 192 & \textbf{\color{blue}0.188} & \textbf{\color{blue}0.152} & \textbf{\color{blue}0.188} & 0.154 & 0.191 & 0.153 & \textbf{\color{red}0.186} & \textbf{\color{red}0.151} & 0.191 & 0.153 & 0.191 & 0.153 & 0.190 & 0.155 & \textbf{\color{blue}0.188} & \textbf{\color{blue}0.152} & 0.195 & 0.169 \\
& 336 & \textbf{\color{blue}0.195} & \textbf{\color{red}0.162} & 0.197 & 0.166 & 0.197 & \textbf{\color{blue}0.163} & \textbf{\color{red}0.192} & \textbf{\color{red}0.162} & 0.197 & 0.164 & 0.208 & 0.173 & 0.199 & 0.167 & \textbf{\color{blue}0.195} & 0.164 & 0.206 & 0.180 \\
& 720 & \textbf{\color{blue}0.206} & \textbf{\color{red}0.183} & 0.207 & 0.186 & 0.216 & 0.190 & \textbf{\color{red}0.205} & \textbf{\color{blue}0.184} & 0.208 & 0.183 & 0.215 & 0.189 & 0.212 & 0.189 & 0.208 & 0.185 & 0.219 & 0.201\\
& AVG & \textbf{\color{blue}0.192} & \textbf{\color{blue}0.160} & 0.193 & 0.162 & 0.198 & 0.163 & \textbf{\color{red}0.190} & \textbf{\color{red}0.159} & 0.195 & 0.160 & 0.199 & 0.165 & 0.196 & 0.164 & \textbf{\color{blue}0.192} & 0.161 & 0.201 & 0.176\\
\midrule
\multirow{5}{*}{\begin{sideways}Air\end{sideways}} 
& 96  & \textbf{\color{red}0.441} & \textbf{\color{red}0.307} & 0.451 & 0.314 & 0.468 & 0.329 & 0.447 & \textbf{\color{blue}0.308} & 0.441 & 0.305 & \textbf{\color{blue}0.443} & \textbf{\color{red}0.307} & 0.470 & 0.337 & 0.465 & 0.331 & \textbf{\color{red}0.441} & 0.325 \\
& 192 & 0.461 & \textbf{\color{red}0.325} & 0.472 & 0.331 & 0.481 & 0.340 & 0.467 & 0.328 & 0.465 & 0.331 & \textbf{\color{red}0.451} & \textbf{\color{blue}0.329} & 0.485 & 0.349 & 0.477 & 0.341 & \textbf{\color{blue}0.460} & 0.338 \\
& 336 & \textbf{\color{blue}0.463} & \textbf{\color{red}0.339} & 0.468 & 0.342 & 0.485 & 0.351 & \textbf{\color{red}0.461} & \textbf{\color{red}0.339} & 0.473 & 0.348 & 0.468 & 0.342 & 0.499 & 0.363 & 0.484 & 0.353 & \textbf{\color{red}0.461} & \textbf{\color{blue}0.341} \\
& 720 & \textbf{\color{red}0.470} & \textbf{\color{red}0.358} & 0.492 & 0.379 & 0.501 & 0.386 & 0.482 & 0.367 & 0.488 & 0.374 & 0.488 & 0.369 & 0.516 & 0.401 & 0.504 & 0.392 & \textbf{\color{blue}0.474} & \textbf{\color{blue}0.359}\\
& AVG & \textbf{\color{red}0.459} & \textbf{\color{red}0.332} & 0.471 & 0.342 & 0.484 & 0.352 & 0.464 & \textbf{\color{blue}0.336} & 0.467 & 0.340 & 0.463 & 0.337 & 0.493 & 0.363 & 0.483 & 0.354 &\textbf{\color{red}0.459} & 0.341 \\
\midrule
\multirow{5}{*}{\begin{sideways}Weather\end{sideways}} 
& 96  & 0.167 &0.212 & 0.175 & 0.215 & \textbf{\color{blue}0.165} & \textbf{\color{blue}0.208} &0.175 &0.217 & 0.170 & 0.215 & \textbf{\color{red}0.158} & \textbf{\color{red}0.203} & 0.175 & 0.216 & 0.176 & 0.217 & 0.197 & 0.259 \\
& 192 & 0.217 & 0.255 & 0.223 & 0.257 & \textbf{\color{blue}0.215} & \textbf{\color{blue}0.254} &0.219 & 0.256 & 0.214 & 0.253 & \textbf{\color{red}0.211} & \textbf{\color{red}0.252} & 0.225 & 0.257 & 0.223 & 0.257 & 0.238 & 0.299 \\
& 336 & 0.274 & \textbf{\color{blue}0.294} & 0.278 & 0.297 & \textbf{\color{blue}0.273} & 0.297 &0.275 & 0.298 & 0.266 & 0.292 & \textbf{\color{red}0.267} & \textbf{\color{red}0.292} & 0.280 & 0.298 & 0.277 & 0.296 & 0.282 & 0.331 \\
& 720 & \textbf{\color{blue}0.351} & \textbf{\color{red}0.345} & 0.355 & 0.349 & 0.354 &0.349 &0.353 & 0.349& 0.347 & 0.345 & \textbf{\color{blue}0.351} & \textbf{\color{blue}0.346} & 0.358 & 0.350 & 0.354 & 0.348 & \textbf{\color{red}0.350} & 0.388 \\
& AVG & \textbf{\color{blue}0.252} & \textbf{\color{blue}0.276} & 0.258 & 0.280 & \textbf{\color{blue}0.252} & 0.277 & 0.256 & 0.280 & 0.249 & 0.276 & \textbf{\color{red}0.247}& \textbf{\color{red}0.273} & 0.260 & 0.280 & 0.258 & 0.280 & 0.267 & 0.319 \\
\midrule
\midrule
\multicolumn{2}{c|}{$1^{st}/2^{st}$\text{Count}}& \textbf{\color{red}31} & \textbf{\color{red}34} &{4} & {3} & 5 & 4 & 14 & 12 & {-} & {-} & 12 & 13 & 1 & 1 & \textbf{\color{blue}16} & \textbf{\color{blue}14} & 9 & 3 \\
\bottomrule[1.5pt]
\end{tabular}}
\end{adjustbox}
\end{table*}
 
\begin{remark}
    {\color{mypurple}\textbf{TimeMachine}} is an ensemble learning model that falls outside the scope of our comparison. However, we still include the experimental results in the table for future reference in research. It is worth mentioning that {\color{mypurple}\textbf{TimeMachine}} has several times the number of parameters and complexity compared to Time-SSM, yet it still performs worse than Time-SSM on five datasets.
\end{remark}

\subsection{Hyper-parameter}
\begin{figure}[h!]
	\centering
    {\includegraphics[width=.45\columnwidth]{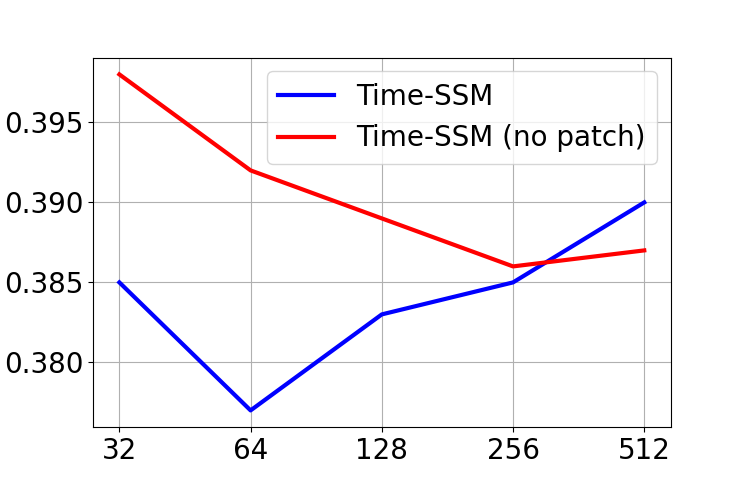}}
	\caption{\cite{hu2023twins} Hyper-parameter analysis of MSE in ETTh1(96-96)}
 \label{hyper}
\end{figure}
As shown in Figure \ref{hyper}, we conducted a hyperparameter analysis of the Time-SSM state spectral space dimensions. It can be observed that for the chunked version of Time-SSM, 64 dimensions yield optimal results, which further validates our inference in Section 4.3 (Representation Ability). However, the absence of patch operations necessitates higher spatial dimensions for representation.

\begin{table*}[h!]
\centering
\caption{Additional Baselines.}\label{extra table}
\setlength{\tabcolsep}{10pt}
\begin{adjustbox}{width=1\columnwidth, center}
\resizebox{1\columnwidth}{!}{
\begin{tabular}{p{0.3cm}|c|cc|cc|cc|cc|cc|cc|cc|cc|cc}

\toprule[1.5pt]
\multicolumn{2}{c}{\multirow{2}{*}{\large{Model}}}&\multicolumn{2}{c}{\cellcolor{mycolor4}Time-SSM}&\multicolumn{2}{c}{\cellcolor{mycolor4}GPT4TS}&\multicolumn{2}{c}{\cellcolor{mycolor4}TimesNet}&\multicolumn{2}{c}{\cellcolor{mycolor4}Crossformer} &\multicolumn{2}{c}{\cellcolor{mycolor4}Scaleformer}&\multicolumn{2}{c}{\cellcolor{mycolor4}TiDE}&\multicolumn{2}{c}{\cellcolor{mycolor4}NLinear}&\multicolumn{2}{c}{\cellcolor{mycolor4}Stationary}&\multicolumn{2}{c}{\cellcolor{mycolor4}FEDformer}\\
\multicolumn{2}{c}{} & \multicolumn{2}{c}{\cellcolor{mycolor4}{(Ours)}} & \multicolumn{2}{c}{\cellcolor{mycolor4}\cite{zhou2023one}}&\multicolumn{2}{c}{\cellcolor{mycolor4}\cite{wu2023timesnet}} & \multicolumn{2}{c}{\cellcolor{mycolor4}\cite{zhang2023crossformer}} & \multicolumn{2}{c}{\cellcolor{mycolor4}\cite{shabani2023scaleformer}} & \multicolumn{2}{c}{\cellcolor{mycolor4}\cite{das2023tide}} & \multicolumn{2}{c}{\cellcolor{mycolor4}\cite{zeng2023transformers}} & \multicolumn{2}{c}{\cellcolor{mycolor4}\cite{liu2022non}} & \multicolumn{2}{c}{\cellcolor{mycolor4}\cite{zhou2022fedformer}} \\
\cmidrule(l){1-2}\cmidrule(l){3-4}\cmidrule(l){5-6}\cmidrule(l){7-8}\cmidrule(l){9-10}\cmidrule(l){11-12}\cmidrule(l){13-14}\cmidrule(l){15-16}\cmidrule(l){17-18}\cmidrule(l){19-20}
\multicolumn{2}{c}{Metric}&MSE & MAE & MSE &MAE & MSE &MAE & MSE & MAE & MSE & MAE & MSE & MAE & MSE & MAE & MSE & MAE & MSE & MAE \\
\midrule
\midrule
\multirow{5}{*}{\begin{sideways}ETTh1\end{sideways}} 
& 96 & \textbf{\color{blue}0.377} & \textbf{\color{red}0.394} & 0.398 & 0.424 & 0.384 & 0.402 &0.391 & 0.417 & 0.396 & 0.440 & 0.479 & 0.464 & 0.386 & \textbf{\color{blue}0.400} & 0.513 & 0.419 & \textbf{\color{red}0.376} & 0.419 \\
& 192 & 0.423 & \textbf{\color{red}0.424} & 0.441 & 0.436 & 0.436 & \textbf{\color{blue}0.429} & 0.449 & 0.452 & 0.434 & 0.460 & 0.524  & 0.490  & \textbf{\color{red}0.400} & 0.430 & 0.533 & 0.505 & \textbf{\color{blue}0.420} & 0.448 \\
& 336 & 0.466 & \textbf{\color{red}0.437} & 0.492 & 0.466 & 0.491 & 0.469 & 0.510 & 0.489 & \textbf{\color{blue}0.462} & 0.476 & 0.566  & 0.513  & 0.480 & \textbf{\color{blue}0.443} & 0.586 & 0.535 & \textbf{\color{red}0.460} & 0.465 \\
& 720 & \textbf{\color{red}0.452} & \textbf{\color{red}0.448} & 0.487 & 0.483 & 0.521 & 0.500 & 0.594 & 0.567 & 0.494 & 0.500 & 0.595  & 0.558  & \textbf{\color{blue}0.486} & \textbf{\color{blue}0.472} & 0.644 & 0.616 & 0.505 & 0.507\\
& AVG & \textbf{\color{red}0.374} & \textbf{\color{red}0.399} & 0.457 & \textbf{\color{blue}0.450} & 0.458 & \textbf{\color{blue}0.450} & 0.486 & 0.481 & 0.447 & 0.469 & 0.541  & 0.506& \textbf{\color{blue}0.438} & 0.436 & 0.569 & 0.519 & 0.440 & 0.460\\
\midrule
\multirow{5}{*}{\begin{sideways}ETTh2\end{sideways}} 
& 96 &\textbf{\color{red} 0.290} & \textbf{\color{red}0.341} & \textbf{\color{blue}0.312} & \textbf{\color{blue}0.360} & 0.340 & 0.374 & 0.641 & 0.549 & 0.364 & 0.407 & 0.400 & 0.440 & 0.324 & 0.369 & 0.476 & 0.458 & 0.358 & 0.397\\
& 192 & \textbf{\color{red}0.368} & \textbf{\color{red}0.387} & \textbf{\color{blue}0.387} & \textbf{\color{blue}0.405} & 0.402 & 0.414 & 0.896 & 0.656 & 0.466 & 0.458 & 0.528 & 0.509 & 0.413 & 0.421 & 0.512 & 0.493 & 0.429 & 0.439\\
& 336 & \textbf{\color{red}0.396} & \textbf{\color{red}0.402} & \textbf{\color{blue}0.424} & \textbf{\color{blue}0.437} & 0.452 & 0.452 & 0.936 & 0.690 & 0.479 & 0.476 & 0.627 & 0.560 & 0.455 & 0.455 & 0.552 & 0.551 & 0.496 & 0.487 \\
& 720 & \textbf{\color{blue}0.449} & \textbf{\color{red}0.440} & \textbf{\color{red}0.433} & \textbf{\color{blue}0.453} & 0.462 & 0.468 & 1.390 & 0.863 & 0.487 & 0.492 & 0.874 & 0.679 & 0.457 & 0.466 & 0.562 & 0.560 & 0.463 & 0.474\\
& AVG & \textbf{\color{red}0.374} & \textbf{\color{red}0.399} & \textbf{\color{blue}0.389} & \textbf{\color{blue}0.414} & 0.414 & 0.427 & 0.436 & 0.453 & 0.445 & 0.458 & 0.607 & 0.547 & 0.412 & 0.428 & 0.526 & 0.516 & 0.437 & 0.449 \\
\midrule
\multirow{5}{*}{\begin{sideways}ETTm1\end{sideways}} 
& 96 & \textbf{\color{red}0.329} & \textbf{\color{red}0.365} & \textbf{\color{blue}0.335} & \textbf{\color{blue}0.369} & 0.338 & 0.375 & 0.366 & 0.400 & 0.355 & 0.398 & 0.364  & 0.387  & 0.339 & \textbf{\color{blue}0.369} & 0.386 & 0.398 & 0.379 & 0.419 \\
& 192 & \textbf{\color{red}0.370} & \textbf{\color{red}0.379} & \textbf{\color{blue}0.374} & \textbf{\color{blue}0.385} & \textbf{\color{blue}0.374} & 0.387 & 0.396 & 0.414 & 0.428 & 0.455 & 0.400  & 0.402  & 0.379 & 0.386 & 0.460 & 0.445 & 0.426 & 0.440 \\
& 336 & \textbf{\color{red}0.396} & \textbf{\color{red}0.402} & \textbf{\color{blue}0.407} & \textbf{\color{blue}0.406} & 0.410 & 0.411 & 0.439 & 0.443 & 0.524 & 0.487 & 0.428  & 0.424  & 0.411 & 0.407 & 0.495 & 0.464 & 0.445 & 0.459 \\
& 720 & \textbf{\color{red}0.449} & \textbf{\color{red}0.440} & \textbf{\color{blue}0.469} & \textbf{\color{blue}0.442} & 0.478 & 0.450 & 0.540 & 0.509 & 0.558 & 0.517 & 0.486  & 0.460  & 0.478 & 0.442 & 0.584 & 0.515 & 0.543 & 0.491 \\
& AVG & \textbf{\color{red}0.386} &\textbf{\color{red} 0.396 }& \textbf{\color{blue}0.396} & \textbf{\color{blue}0.401} & 0.400 & 0.406 & 0.435 & 0.442 & 0.466 & 0.464 & 0.420  & 0.418  & 0.402 & \textbf{\color{blue}0.401} & 0.481 & 0.456 & 0.448 & 0.452 \\
\midrule
\multirow{5}{*}{\begin{sideways}ETTm2\end{sideways}} 
& 96 & \textbf{\color{red}0.176} & \textbf{\color{blue}0.260} & 0.190 & 0.275 & 0.187 & 0.267 & 0.273 & 0.346 & 0.182 & 0.275 & 0.207 & 0.305 & \textbf{\color{blue}0.177} & \textbf{\color{red}0.257} & 0.192 & 0.274 & 0.203 & 0.287 \\
& 192 & \textbf{\color{blue}0.246} & \textbf{\color{blue}0.305} & 0.253 & 0.313 & 0.249 & 0.309 & 0.350 & 0.421 & 0.251 & 0.318 & 0.290 & 0.364 & \textbf{\color{red}0.241} & \textbf{\color{red}0.297} & 0.280 & 0.339 & 0.269 & 0.328 \\
& 336 & \textbf{\color{blue}0.305} & \textbf{\color{blue}0.344} & 0.321 & 0.360 & 0.321 & 0.351 & 0.474 & 0.505 & 0.340 & 0.375 & 0.377 & 0.422 & \textbf{\color{red}0.302} & \textbf{\color{red}0.337} & 0.334 & 0.361 & 0.325 & 0.336 \\
& 720 & \textbf{\color{blue}0.406} & 0.405 & 0.411 & 0.406 & 0.408 & \textbf{\color{blue}0.403} & 1.347 & 0.812 & 0.435 & 0.433 & 0.558 & 0.524 & \textbf{\color{red}0.405} & \textbf{\color{red}0.396} & 0.417 & 0.413 & 0.421 & 0.415\\
& AVG & \textbf{\color{blue}0.283} & \textbf{\color{blue}0.328} & 0.294 & 0.339 & 0.291 & 0.333 & 0.611 & 0.521 & 0.302 & 0.350 & 0.358 & 0.404 & \textbf{\color{red}0.281} & \textbf{\color{red}0.322} & 0.306 & 0.347 & 0.305 & 0.349 \\
\midrule
\multirow{5}{*}{\begin{sideways}Exchange\end{sideways}} 
& 96 & \textbf{\color{red}0.083} & \textbf{\color{red}0.202} & 0.091 & 0.212 & 0.107 & 0.234 & 0.256 & 0.367 & 0.155 & 0.285 & 0.094 & 0.218 & \textbf{\color{blue}0.089} & \textbf{\color{blue}0.208} & 0.111 & 0.237 & 0.148 & 0.278 \\
& 192 & \textbf{\color{red}0.170} & \textbf{\color{red}0.295} & 0.183 & 0.304 & 0.226 & 0.344 & 0.469 & 0.508 & 0.274 & 0.384 & 0.184 & 0.307 & \textbf{\color{blue}0.180} & \textbf{\color{blue}0.301} & 0.219 & 0.335 & 0.271 & 0.315 \\
& 336 & 0.334 & 0.418 & \textbf{\color{blue}0.328}& \textbf{\color{blue}0.417} & 0.367 & 0.448 & 0.901 & 0.741 & 0.452 & 0.498 & 0.349 & 0.431 & \textbf{\color{red}0.320} & \textbf{\color{red}0.409} & 0.421 & 0.476 & 0.460 & 0.427 \\
& 720 & \textbf{\color{red}0.824} & \textbf{\color{red}0.677} & 0.880 & 0.704 & 0.964 & 0.746 & 1.398 & 0.965 & 1.172 & 0.839 & 0.852 & 0.698 & \textbf{\color{blue}0.832} & \textbf{\color{blue}0.685} & 1.092 & 0.769 & 0.964 & 0.746 \\
& AVG & \textbf{\color{red}0.352} & \textbf{\color{red}0.398} & 0.371 & 0.409 & 0.416 & 0.443 & 0.756 & 0.645 & 0.513 & 0.502 & 0.370 & 0.413 & \textbf{\color{blue}0.355} & \textbf{\color{blue}0.401} & 0.461 & 0.454 & 0.519 & 0.429 \\
\midrule
\multirow{5}{*}{\begin{sideways}Weather\end{sideways}}
& 96 & \textbf{\color{blue}0.167} & \textbf{\color{blue}0.212} & 0.203 & 0.244 & 0.172 & 0.220 & \textbf{\color{red}0.164} & 0.232 & 0.288 & 0.365 & 0.202 & 0.261 & 0.168 & \textbf{\color{red}0.208} & 0.173 & 0.223 & 0.217 & 0.296 \\
& 192 & \textbf{\color{blue}0.217} &\textbf{\color{red} 0.255} & 0.247 & 0.277 & 0.219 & \textbf{\color{blue}0.261} & \textbf{\color{red}0.211} & 0.276 & 0.368 & 0.425 & 0.242 & 0.298 & \textbf{\color{blue}0.217} & \textbf{\color{red}0.255} & 0.245 & 0.285 & 0.219 & \textbf{\color{blue}0.261} \\
& 336 & 0.274 & \textbf{\color{blue}0.294} & 0.297 & 0.311 & 0.280 & 0.306 & \textbf{\color{blue}0.269} & 0.327 & 0.447 & 0.469 & 0.287 & 0.335 & \textbf{\color{red}0.267} & \textbf{\color{red}0.292} & 0.321 & 0.338 & 0.339 & 0.380 \\
& 720 & \textbf{\color{red}0.351} & \textbf{\color{red}0.345} & 0.368 & \textbf{\color{blue}0.356} & 0.365 & 0.359 & \textbf{\color{blue}0.355} & 0.404 & 0.640 & 0.574 & \textbf{\color{blue}0.355} & 0.404 & \textbf{\color{red}0.351} & 0.386 & 0.414 & 0.410 & 0.403 & 0.428 \\
& AVG &0.252 & \textbf{\color{blue}0.276} & 0.279 & 0.297 & 0.259 & 0.287 & \textbf{\color{red}0.250} & 0.310 & 0.436 & 0.458 & 0.271 & 0.320 & \textbf{\color{blue}0.251} & \textbf{\color{red}0.275} & 0.288 & 0.314 & 0.309 & 0.360 \\
\midrule
\midrule
\multicolumn{2}{c|}{$1^{st}/2^{st}$\text{Count}}& \textbf{\color{red}25} & \textbf{\color{red}28} & 11 & 13 & 1 & 4 & 5 & - &1 & - &1 & - & \textbf{\color{blue}17} & \textbf{\color{blue}19} & - & - & 3 & - \\
\bottomrule[1.5pt]
\end{tabular}}
\end{adjustbox}
\end{table*}

\section{Discussion}
Although some recent models, such as Pathformer \cite{chen2024pathformer} and Time-LLM \cite{jin2023time}, have surpassed Time-SSM in performance, we believe that operations such as multi-scale patching and large model enhancement can be integrated into Time-SSM. This integration can lead to the emergence of Time-multipatch-SSM or even Time-LargeSSM. It is important to note that the focus of this paper is to explore the feasibility of SSM in the field of time series prediction, with a commitment to the simplest SSM framework and smallest model size. Furthermore, we are the first to provide a detailed theoretical basis for applying SSM to TSF tasks in our article. We provide guidance for future studies on how to apply SSM effectively.

\section{Limitation}
\label{limitation}
This paper provides comprehensive theoretical guidance for the application of State Space Models (SSMs) to time series data. However, there are still areas that warrant further exploration. For instance, while this paper intuitively explains the S4-real initialization method as a rough approximation of dynamics, it lacks rigorous mathematical proof. Furthermore, experimental results demonstrate a significant performance improvement of the Hippo-LegP method compared to Hippo-LegT. As this framework is built upon a finite measure approximation, it remains unclear whether it can be extended to Hippo-LegS, which employs an exponential decay approximation.



\end{document}